\documentclass{article}

\PassOptionsToPackage{numbers, compress}{natbib}

\usepackage[main, final]{neurips_2026}

\usepackage[utf8]{inputenc} 
\usepackage[T1]{fontenc}    
\usepackage{hyperref}       
\usepackage{url}            
\usepackage{booktabs}       
\usepackage{amsfonts}       
\usepackage{nicefrac}       
\usepackage{microtype}      
\usepackage{xcolor}         
\usepackage{amsmath}
\usepackage{amssymb}
\usepackage{comment}
\usepackage{algorithm}
\usepackage{algorithmicx}
\usepackage{algpseudocode}
\usepackage{amsthm}
\usepackage{graphicx}
\usepackage{booktabs}
\usepackage{multirow}
\usepackage{placeins}
\title{RheoSampling: Resolving the One-Hot Dilemma in Stochastic Dynamic-Tree Speculative Decoding}

\author{
Qiao Hu$^{1}$\thanks{Equal Contribution.} \quad Yepeng Weng$^{2,3*}$\thanks{Corresponding author.} \quad Bo Zhang$^{4,5}$ \quad Takehisa Yairi$^{2}$\\
$^1$ National Center for Mathematics and Interdisciplinary Sciences (NCMIS), AMSS, CAS \\
$^2$ The University of Tokyo \quad $^3$Lenovo AI Technology Center \\ 
$^4$ SKLMS and AMSS, Chinese Academy of Sciences \\
$^5$ School of Mathematical Sciences, University of Chinese Academy of Sciences \\
\texttt{huqiao2020@amss.ac.cn, yweng@g.ecc.u-tokyo.ac.jp}}

\begin{document}

\maketitle

\begin{abstract}
Speculative decoding accelerates LLM inference by drafting multiple tokens in parallel, with tree-based methods further improving efficiency through structured hierarchies. Dynamic-tree methods such as EAGLE-3 achieve excellent performance under greedy decoding via deterministic top-$K$ expansion and global pruning. However, in stochastic decoding ($T>0$), this mechanism collapses the draft distribution into one-hot probabilities, causing a severe drop in acceptance rate. This exposes an apparent dilemma: dynamic-tree methods sacrifice stochastic sampling to preserve context-aware topology, while static-tree methods preserve stochastic sampling with context-agnostic structures. The issue arises because the same probability distribution is used for two conflicting tasks: constructing the tree and verifying the tokens. This coupling makes direct injection of randomness extremely challenging, as we are faced with a complex stochastic process.
We resolve this by decoupling these two roles: RheoSampling assigns a token sampled from the draft distribution a \textit{proxy probability} (for tree expansion and pruning) alongside its \textit{true sampling probability} (for verification). Specifically, we inject a sampled token among the deterministic top-$K$ slots and treat it with different probabilities in the construction and verification process, making RheoSampling the first dynamic-tree method with both context-aware top-$K$ construction and stochastic sampling while maintaining losslessness. We establish the lossless guarantee through a novel equivalence-class analysis that compresses the stochastic tree space into tractable classes. An OT-based verification strategy and a sparse draft mechanism ensure that theoretical gains translate into practical efficiency.
Experiments across diverse LLMs and benchmarks demonstrate consistent improvements in acceptance rate and speedup over state-of-the-art dynamic tree methods. This framework may provide a template for analyzing other complex stochastic tree structures.
\end{abstract}

\section{Introduction}

Large Language Models (LLMs)~\cite{GPT-4, GPT-5, llama3, qwen3} have demonstrated remarkable capabilities across diverse tasks, yet their autoregressive decoding nature incurs high inference latency. Speculative decoding \cite{LeviathanSpec, ChenSpec} mitigates this bottleneck by employing a lightweight draft model to predict multiple future tokens in parallel, which are then verified by the target model in a single forward pass. Early approaches draft tokens in a single chain. To further increase the probability of matching the target distribution per step, recent works have shifted towards tree-based speculative decoding \cite{SpecInfer, ChenSequoia, optree, Medusa, LiEAGLE, LiEAGLE2, li2025eagle3}, which organizes multiple candidate sequences as a tree.

While initial tree methods employed static or fixed tree structures, state-of-the-art approaches such as EAGLE-2/3 \cite{LiEAGLE2, li2025eagle3} adopt dynamic tree construction, where the draft topology is contextually adapted at each step. These methods typically follow an \textit{expand-then-rerank} paradigm: they greedily expand the draft tree by selecting tokens with the highest path probabilities, then rerank and prune candidates to fit a global verification budget.
Notably, this entire pipeline is essentially \textit{deterministic}: once the draft distribution is computed, the tree topology and selected candidates are fully determined, with no randomness involved.
Under greedy decoding ($T=0$), this deterministic paradigm is highly effective. However, under stochastic decoding ($T>0$), this mechanism collapses the draft distribution into degenerate one-hot probabilities during verification, causing the under-exploration of tail distribution and a severe acceptance-rate drop. This exposes an apparent dilemma: dynamic-tree methods sacrifice stochastic sampling to preserve context-aware topology, while static-tree methods preserve stochastic sampling with context-agnostic structures.

In this paper, we propose \textbf{RheoSampling} (Rheostat Sampling), a hybrid paradigm that resolves this challenge. The root difficulty is not merely introducing randomness, but handling a complex stochastic process: the tree topology and the sampled token identities are tightly coupled, yielding a probability space too vast for direct analysis. We overcome this through two synergistic innovations: a \textit{dual-identity} probability design that decouples tree construction from verification, and a novel equivalence-class analysis that compresses the dynamic tree space into tractable classes. Specifically, we insert a single sampled token into the top-$K$ candidate pool, immediately following the top-$m$ deterministic tokens. We then apply a \textit{dual treatment} to the sampled token: during dynamic tree construction, it is assigned a \textit{proxy probability} to compete with deterministic candidates for expansion and survival; during verification, it is evaluated using its \textit{true sampling probability}, guaranteeing losslessness. By decoupling the probabilities for tree construction and verification, RheoSampling injects stochasticity into dynamic drafting and materializes it in the final tree.

To further improve the acceptance rate of RheoSampling, we design a simple yet efficient verification algorithm based on Optimal Transport (OT) and prove its losslessness. Furthermore, we introduce a sparse draft distribution strategy to reduce computational overhead. Extensive experiments across diverse LLMs and benchmarks demonstrate that RheoSampling consistently outperforms top-$K$-based dynamic trees in both acceptance rate and wall-clock speedup in standard stochastic decoding.

Our contributions are summarized as follows:
\begin{itemize}
    \item \textbf{Dual-identity decoupling for the one-hot dilemma.} RheoSampling gives a stochastically sampled token two independent probability identities: a \textit{proxy probability} for tree construction and a \textit{true sampling probability} for verification. This decoupling resolves the conflict between dynamic top-$K$ construction and stochastic sampling, making RheoSampling the first method to achieve both with rigorous losslessness.
    \item \textbf{Equivalence-class analysis for losslessness.} We develop an equivalence-class methodology that reduces the stochastic tree space to a tractable form, yielding the first rigorous proof of losslessness for stochastic dynamic trees. This analytical framework may provide a template for analyzing other complex stochastic tree structures.
    \item \textbf{Efficient algorithmic design.} We design a verification algorithm for RheoSampling based on Optimal Transport, and introduce a sparse draft distribution that enables efficient sampling over large vocabularies, translating theoretical gains directly into wall-clock speedup.
    \item \textbf{Comprehensive empirical validation.} Extensive experiments on various LLMs and tasks show that RheoSampling consistently achieves higher acceptance rates and speedups over the EAGLE-3 baseline, validating its robustness and generality.
\end{itemize}

\section{Preliminaries}

\subsection{Tree-Based Speculative Decoding}
\label{sec:tree-sd}

Tree-based speculative decoding organizes multiple candidate sequences into a tree structure $\mathcal{T}$, allowing the target model to verify diverse drafting paths in parallel. Early approaches employed static tree topologies constructed heuristically (\emph{e.g.,} SpecInfer \cite{SpecInfer}, EAGLE-1 \cite{LiEAGLE}). To further improve tree quality and hit rates, recent works have shifted to dynamic, context-aware construction (e.g., EAGLE-2/3 \cite{LiEAGLE2, li2025eagle3}), where the tree topology adapts to the context. \textit{Unless otherwise specified, we refer to EAGLE as the dynamic tree variant throughout this paper.}

EAGLE employs an \textit{expand-then-rerank} paradigm for dynamic tree construction. At each decoding step, the draft model performs parallel forward passes on $K$ selected parent nodes to obtain next-token distributions. For each parent, it selects the top-$K$ tokens to form a candidate pool, resulting in $K \times K$ leaf nodes per layer. Each candidate's path score is the cumulative probability along its root-to-leaf path. Based on these scores, another top-$K$ selection determines which nodes to expand in the next layer. This process continues until reaching depth $D$, after which the tree is pruned to satisfy a global budget $N$, keeping the nodes with the highest scores. A concrete example of this tree-construction procedure is provided in Appendix~\ref{app:tree-example}.

This mechanism inherently collapses the stochastic draft distribution into a deterministic selection. Specifically, because tokens are exclusively chosen via top-$K$ operations, the actual \textit{proposal distribution} degenerates into a set of point masses. Consequently, the rich long-tail probability information in $q$ is entirely discarded during tree expansion.

\begin{figure}[t]
\centering
\includegraphics[width=0.8\linewidth]{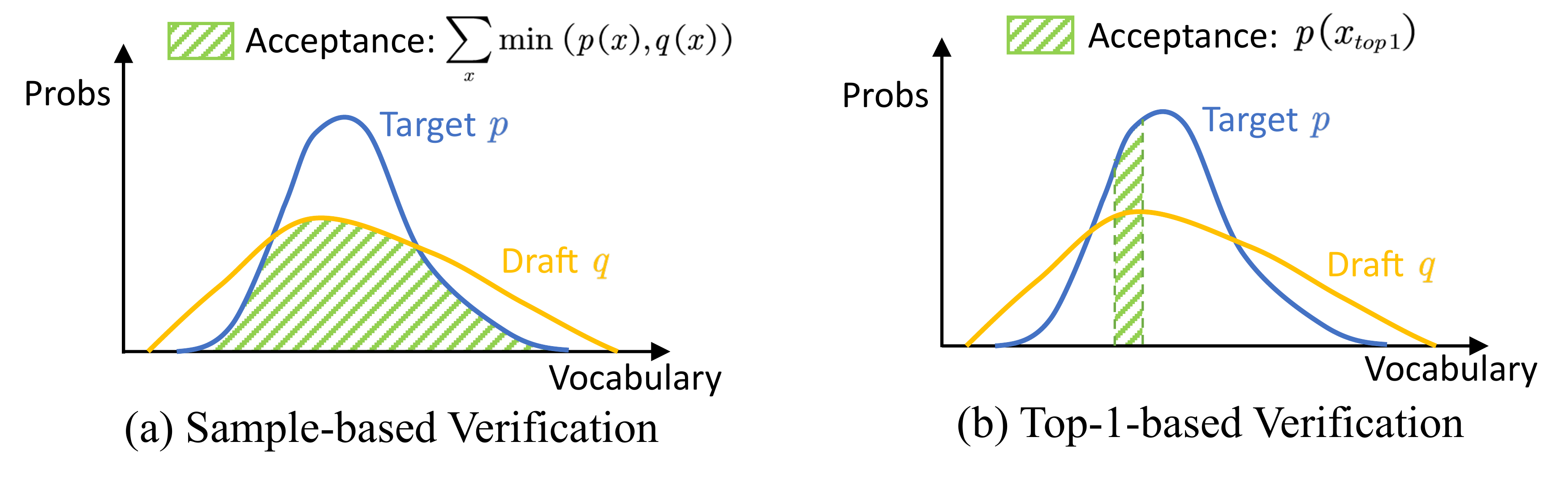}
\caption{Illustration of acceptance rates under different sampling strategies. (a) Sample-based verification: the acceptance rate equals the distribution overlap $\sum_x \min(p(x), q(x))$. (b) Top-$1$-based verification: the draft deterministically selects $x_{top1} = \arg\max_x q(x)$, yielding an acceptance rate of $p(x_{top1})$, which equals the probability mass the target model assigns to this single point.}
\label{fig:acceptance_comparison}
\end{figure}

\subsection{Sampling and Verification Strategies}

To illustrate the differences between sampling strategies, consider the single-draft scenario (Fig.~\ref{fig:acceptance_comparison}). For standard speculative sampling (\emph{i.e.,} rejection sampling) \cite{LeviathanSpec, ChenSpec}, the acceptance rate equals the overlap between the target distribution $p$ and the draft distribution $q$, formally $\alpha^* = \sum_x \min(p(x), q(x))$. In contrast, under top-$1$ sampling, the draft deterministically selects the token with the highest $q(x)$, and the acceptance rate reduces to the target probability assigned to this single candidate, which is typically much smaller than $\alpha^*$ at high temperature.

These considerations also exist in multi-draft scenarios. In fact, pure top-$K$ sampling suffers from an intrinsic limitation. Even when the draft model perfectly matches the target, top-$K$ sampling cannot achieve $100\%$ acceptance rate if the target assigns positive probability mass outside the candidate set, violating the \textit{optimal transport property}, which is discussed in Sequoia \cite{ChenSequoia}. Modern multi-draft strategies such as recursive rejection sampling \cite{SpecInfer,YangMCSD,LiEAGLE,JeonRSD} and OT-based \cite{SpecTr,SpecHub,TowardsOptimal,thomas2026global} methods rely on the presence of \textit{at least one stochastically sampled token} within the candidate set. In brief, even a single stochastic token in the candidate pool provides the distributional flexibility for advanced verification algorithms; on the contrary, pure top-$K$ selection inherently lacks this operational space, precluding any non-trivial allocation strategy.

\section{Method}
\label{sec:method}

\subsection{The Incompatibility of Stochastic Sampling and Dynamic Tree Construction}\label{sec:incompatibility}

Replacing top-$K$ with direct sampling seems like the natural way to fix the one-hot collapse. However, this naive substitution fails structurally, because it destroys the foundation that dynamic tree construction relies upon.

\paragraph{Ranking by draft probability and its limits.} 
Dynamic tree methods such as EAGLE-2/3 use draft probability $q(x)$ to rank tokens. Tokens with higher $q(x)$ are selected as parent nodes for expansion, and path scores are computed by multiplying $q(x)$ along the path. Under greedy decoding ($T=0$), this works well because the draft model's top choice is usually aligned with the target one.
However, under non-greedy sampling ($T>0$), whether a token is accepted depends on the ratio $p(x)/q(x)$, not on the size of $q(x)$ itself. A token with small $q(x)$ can still be accepted if $p(x)$ is large enough, and a high $q(x)$ does not necessarily guarantee a high acceptance rate. Thus, $q(x)$ is no longer a reliable indicator of acceptance. If we directly replace top-$K$ selection with stochastic sampling and use $q(x)$ to rank the sampled tokens, the rules for expanding and pruning the tree become unclear.

\paragraph{The deeper trap: breaking losslessness.}
A more essential failure arises when sampled tokens are ranked by their raw probabilities. Because a token's value determines its own rank, its survival through pruning depends on its identity. Conditioning on survival distorts the token's conditional distribution away from the original sampling distribution. Once the verification probability no longer matches this distorted distribution, the lossless property is broken.

\subsection{RheoSampling: Dual-Identity Decoupling for Dynamic Trees}

The naive approach fails because it ties a token's survival in the tree to its own sampling probability. This coupling distorts the conditional distribution of surviving tokens, breaking losslessness. To resolve this, we propose \textbf{RheoSampling}, which decouples these two roles. Its core idea is to assign \textit{two independent probability identities to a single stochastic token}: a proxy probability for tree construction (expansion, reranking and pruning), and its true sampling probability for verification.

\paragraph{Hybrid sampling mechanism.} 
Consider a candidate pool with $K$ slots at each expansion step. RheoSampling allocates these slots as follows: first, deterministically select the top-$m$ tokens with highest draft probabilities (the lead tokens); second, sample \textit{one} representative token $x_s$ from the residual distribution $\tilde{q}$ (the tail beyond top-$m$); finally, fill the remaining $K-m-1$ slots with the highest-ranked tokens from the remaining vocabulary (the fill tokens). 
This yields a mixed candidate set combining high-probability tokens with a single stochastic probe into the distribution tail.

\paragraph{Dual treatment of the sampled token.} 
The critical innovation lies in how $x_s$ is treated. During tree construction, it carries a \textbf{proxy probability} $q_{\text{proxy}}(x_s) = \min\{q(x_m), z\}$, where $z = 1 - \sum_{i=1}^m q(x_i)$ is the residual mass and $q(x_m)$ is the $m$-th deterministic probability. This construction places $x_s$ ahead of all fill tokens within the candidate pool, ensuring its survival through global pruning is \textit{independent of its realized identity}. This is an essential property for preserving the losslessness (see Section \ref{sec:theo}). In particular, when $m=0$, we set the proxy probability to $q(x_1) + \epsilon$, anchoring the sampled token at the first slot; here $+\epsilon$ is purely symbolic and denotes that ties with $q(x_1)$ are broken in favor of the sampled slot. During verification, the token is evaluated using its \textbf{true sampling probability} $\tilde{q}(x_s)$, not the proxy.
This dual treatment ensures the tree is built on controlled, principled estimates that preserve structural quality, while verification remains lossless by respecting the actual sampling distribution.

\paragraph{The rheostat parameter.}
By varying $m$, we control the proxy probability assigned to $x_s$ and the number of deterministic lead tokens preceding it. A smaller $m$ assigns a larger proxy, increasing the sampled token's survival probability and ensuring stochastic exploration materializes in the final tree, but reduces the deterministic backbone. Conversely, a larger $m$ preserves high-quality tree topology yet risks pruning the sampled node, as its proxy weight places it at a lower rank. Thus, $m$ acts as a rheostat, balancing stochastic survival against structural quality.

\paragraph{Efficient implementation via sparse distributions.}
To mitigate computational overhead from sampling over large vocabularies, we employ a sparse draft distribution strategy. We truncate logits to the top-128 entries prior to softmax, setting remaining logits to $-\infty$. This yields a sparse distribution that enables efficient sampling without iterating over the full vocabulary. Losslessness is preserved as long as verification employs the \textit{same} truncated distribution used for drafting, and acceptance rate is maintained because top-128 entries capture most probability mass of the draft distribution (see Appendix \ref{appendix:sparse} for more details).

\subsection{RheoVerification: Better Acceptance via Optimal Transport}
\label{sec:verification}
We develop an associated verification mechanism based on Optimal Transport (OT) for RheoSampling to further improve the acceptance rate, as shown in Algorithm~\ref{alg:verification}. We provide a rigorous proof of its losslessness in Section \ref{sec:theo}, together with an approximate analysis of its advantage over vanilla sequential rejection verification (Appendix~\ref{app:comparison}).

\begin{figure}[h]
\centering
\includegraphics[width=0.9\linewidth]{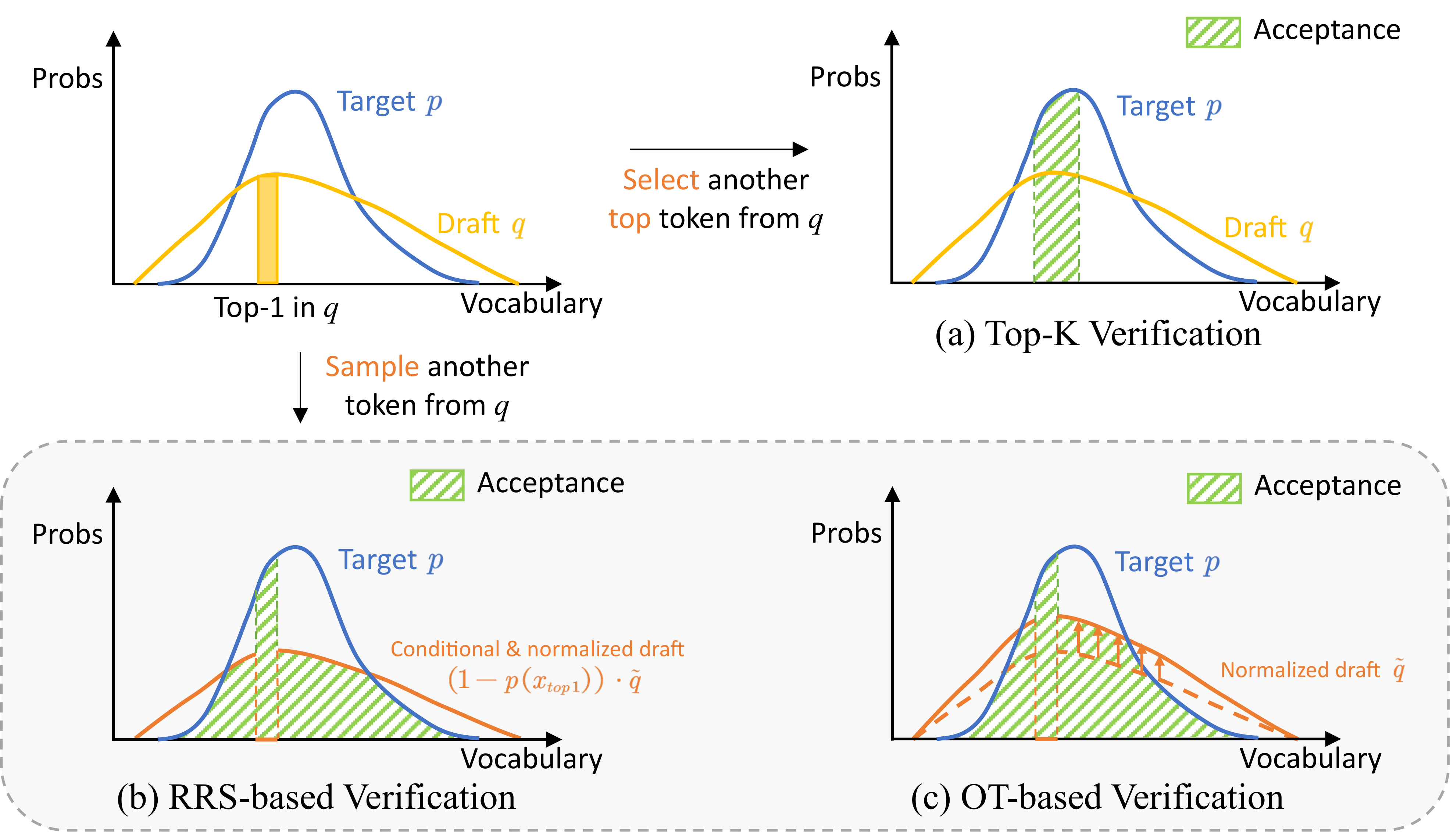}
\caption{Illustrative comparison of multi-draft verification strategies in a simplified two-draft scenario. Starting with the initial top-1 token, the second candidate is either deterministically selected (a) or stochastically sampled (b, c). (a) Top-$K$: Covers only isolated point masses. (b) RRS-based: Explores the distribution tail but verifies sequentially. (c) OT-based (Ours): Globally reallocates the probability mass (using the normalized draft $\tilde{q}$, resulting in larger overlap with target $p$). }
\label{fig:verification}
\end{figure}

\begin{figure}[h]
\centering
\begin{minipage}[t]{0.57\linewidth}

\paragraph{Algorithm overview.} Given the pruned candidate set $\mathcal{C}=\{u_1,\dots,u_n\}$ at a layer, if the sampled token $u_s$ survives reranking ($s\neq -1$), we first verify it independently using its \textit{true} sampling probability $\tilde{q}(u_s)$. With probability $\min(1, p(u_s)/\tilde{q}(u_s))$ the token is accepted and we descend into its subtree. If rejected, we compute the residual distribution $r=\text{norm}(\max(0, p-\tilde{q}))$ and perform recursive rejection over the remaining candidates under $r$. Should every candidate be rejected, we sample once from the final residual. When no sampled token is present ($s=-1$), the verification procedure degenerates to standard top-$K$ verification and relies solely on the target $p$.

\paragraph{Relation to RRS.} A naive alternative for verification is to apply Recursive Rejection Sampling without replacement (RRSw) over the candidates. While RRSw preserves losslessness, it verifies candidates sequentially without globally reallocating the probability mass. As illustrated in Figure~\ref{fig:verification}, pure top-$K$ covers only the isolated point masses. The advantage of our OT-based verification over RRSw lies precisely in achieving higher acceptance rate of the sampled token, effectively utilizing its stochastic flexibility to maximize the overall structural overlap with the target distribution. We provide empirical results in Appendix \ref{appendix:rrs}.

\end{minipage}
\hfill
\begin{minipage}[t]{0.4\linewidth}
\vspace{-17pt}
\begin{algorithm}[H]
\caption{RheoVerification}
\label{alg:verification}
\small
\begin{algorithmic}[1]
\Require Target distribution $p$; tail draft $\tilde{q}$; candidates $\mathcal{C}=\{u_1,\dots,u_n\}$; sample index $s$ ($-1$ if pruned)
\Ensure Next token $u^*$

\If{$s \neq -1$}
    \State $u_s \gets \mathcal{C}[s]$
    \State $\eta \sim \text{Uniform}(0,1)$
    \If{$\eta \leq \min(1, p(u_s)/\tilde{q}(u_s))$}
        \State \textbf{return} $u_s$
    \EndIf
    \State $\mathcal{C} \gets \mathcal{C} \setminus \{u_s\}$
    \State $r \gets \text{norm}(\max(0, p-\tilde{q}))$
\Else
    \State $r \gets p$
\EndIf

\For{$u_i \in \mathcal{C}$ in arbitrary order}
    \State $\eta \sim \text{Uniform}(0,1)$
    \If{$\eta \leq r(u_i)$}
        \State \textbf{return} $u_i$
    \EndIf
    \State $r(u_i) \gets 0$; $r \gets \text{norm}(r)$
\EndFor

\State Sample $u^* \sim r$
\State \textbf{return} $u^*$
\normalsize
\end{algorithmic}
\end{algorithm}
\end{minipage}
\end{figure}

\newtheorem{theorem}{Theorem}
\newtheorem{remark}[theorem]{Remark}
\newtheorem{lemma}[theorem]{Lemma}
\newtheorem{corollary}[theorem]{Corollary}
\newtheorem{proposition}[theorem]{Proposition}

\section{Theoretical Guarantees}
\label{sec:theo}

\textbf{The coupling challenge.} Establishing losslessness for RheoSampling is substantially harder than
for static-tree speculative decoding. 
In a dynamic tree, the stochastic token $Y\sim\tilde{q}$ not only
determines its own value, but also determines the identities of the subsequent fill tokens. These in turn
affect the path scores of deeper nodes and thereby influence which
branches survive global pruning. 
Consequently, the pruned tree $T_R$ is inherently random: both its
topology and the token fillings at each node are stochastic, and the
two sources of randomness are tightly coupled.
We overcome this difficulty via an equivalence-class analysis (full proof in Appendix~\ref{app:proof}; numerical verification in Appendix~\ref{app:numerical}).

\begin{theorem}[Losslessness of RheoSampling]
\label{thm:lossless}
For any
pruning budget $R\ge 1$, any length $L\ge 1$, and any token sequence $\mathbf{Seq}=(v_1,\dots,v_L)$,
let $T_R$ denote the random pruned tree produced by the RheoSampling
draft mechanism followed by global top-$R$ pruning, and $p$ is the target model's distribution.
Then
\begin{equation}
    \mathbb{E}_{T_R}\Bigl[\Pr\bigl(\mathrm{Rheo}(T_R)_{1:L}=\mathbf{Seq}\bigr)\Bigr]
    \;=\; p(\mathbf{Seq}),
\end{equation}
where $\mathrm{Rheo}(T_R)$ is the output sequence of applying
Algorithm~\ref{alg:verification} on $T_R$ layer-by-layer; once the output token at some layer is a leaf of $T_R$ or is not a node of $T_R$, the remaining tokens are drawn autoregressively from $p$ (equivalently, by fresh lossless rounds).
\end{theorem}

\textbf{Proof sketch.}
(i) The proxy of the sampled slot is fixed before sampling and dominates every fill token (Lemma~\ref{lem:local-rank}), so whether the slot survives pruning is independent of the realized $Y$.
(ii) We group unpruned trees into equivalence classes $[\mathcal{U}]_R$ that prune to the same $T_R$.
(iii) We induct on $R$. Within a class $[\mathcal{U}]_k$, the $(k{+}1)$-th admitted node always occupies the same slot (Lemma~\ref{lem:slot}). If this slot is deterministic, it merely adds a candidate to sequential rejection and leaves the output distribution unchanged; if it is the sampled slot, the class splits by the realized token $v$ with proportions $\tilde{q}(v)$, and averaging Algorithm~\ref{alg:verification} over $v\sim\tilde{q}$ recovers exactly $p$ at that layer. Each increment of $R$ thus preserves the output distribution, and the base case $R=1$ reduces to sampling from $p$.

\textbf{A necessary bound on the proxy.} 
Independence of the proxy from the realized sampled token is necessary but not sufficient: the proxy must also satisfy the ranking bound $q_{\mathrm{proxy}}\ge q(x_{m+1})$ (Lemma~\ref{lem:local-rank}). 
Otherwise the survival of the sampled slot depends on the sampled token, and conditioning on survival distorts the conditional distribution away from $\tilde{q}$.
Consider a vocabulary $\{A,B,C\}$ with draft distribution  $\{0.5,\,0.3,\,0.2\}$. 
Table~\ref{tab:proxy} shows the rankings under two proxy choices with $K=3$, $m=1$.
\begin{itemize}
    \item Rheo-proxy probability: $q_{\mathrm{proxy}}=\min\{q(A),1-q(A)\}=0.5$. The sampled token always sits at rank $2$ regardless of the sampled token $Y$, so its survival is independent of $Y$. 

    \item Adhoc-proxy probability: $q_{\mathrm{proxy}}=0.25$. The rank of $Y$ varies with the fill tokens. Conditioned on survival, $Y$ is forced to be $B$ with probability $1$, rather than $\tilde{q}(B)=0.6$. Algorithm~\ref{alg:verification} then evaluates $Y$ against the wrong base probability, breaking losslessness.
\end{itemize}

\begin{table}[h]
\centering
\caption{Candidate-pool rankings and survival under two proxy choices.}
\label{tab:proxy}
\small
\begin{tabular}{@{}ccccl@{}}
\toprule
\textbf{Proxy} & \textbf{Sampled $Y$} & \textbf{Fill tokens} & \textbf{Pool ranking ($R=2$)} & \textbf{Survival tokens} \\
\midrule
\multirow{2}{*}{Rheo: $0.5$} 
  & $\tilde q(B)=0.6$ & $C$ & $\boxed{A(0.5)\ge Y(0.5)} \ge C(0.2)$ & $A, Y\leftarrow B$ \\
  & $\tilde q(C)=0.4$ & $B$ & $\boxed{A(0.5)\ge Y(0.5)} \ge B(0.3)$ & $A, Y\leftarrow C$ \\
\midrule
\multirow{2}{*}{Adhoc: $0.25$} 
  & $\tilde q(B)=0.6$ & $C$ & $\boxed{A(0.5)\ge Y(0.25)} \ge C(0.2)$ & $A, Y\leftarrow B$ \\
  & $\tilde q(C)=0.4$ & $B$ & $\boxed{A(0.5)\ge B(0.3)} \ge Y(0.25)$ & $A, B$ \\
\bottomrule
\end{tabular}
\end{table}

\begin{remark}
Table~\ref{tab:proxy} shows that \textbf{Rheo-proxy} design fixes the sampled slot at rank $m+1$ within a single pool, making the local sibling order frozen. 
Yet across the full tree, the realized token $Y$ still enters the path scores of all descendants, and this in turn determines which branches survive global pruning. 
Hence, the \textbf{pruned tree topology} remains an endogenous random variable coupled with the sampled tokens, and a rigorous proof must account for this compounding randomness (Appendix~\ref{app:proof}).
\end{remark}

We derive a closed-form per-layer acceptance rate (Theorem~\ref{thm:rheo-acc}); the proof and a comparison with the natural RRSw baseline (which places the sampled token later) are deferred to Appendices~\ref{app:acceptance}--\ref{app:comparison}.
As shown in Appendix~\ref{app:comparison}, under the approximation of Eq.~\eqref{eq:approx},
Rheo dominates RRSw, which directly motivates the
stochastic-first strategy of Algorithm~\ref{alg:verification}; this dominance is not guaranteed for every pair $(p,q)$ without the approximation (Appendix~\ref{limitations}).

\begin{theorem}[Single-layer acceptance rate]
\label{thm:rheo-acc}
Fix a parent prefix and a local pool size $n$ that is determined before sampling (i.e., not conditioned on which candidates survive global pruning). Given the resulting candidate pool $\mathcal{C}$ with $|\mathcal{C}|=n$ and $s\neq -1$ in Algorithm~\ref{alg:verification}, let $\mathrm{Top}_n$ be the $n$ highest-$q$ tokens and $x_n$ the $n$-th. The per-layer acceptance rate, averaged over $Y\sim\tilde{q}$ and the randomness of Algorithm~\ref{alg:verification}, is
\begin{align}\label{eq:rheo-acc}
    \mathcal{A}_{\text{Rheo}}
    \;&=\sum_{v\in\mathrm{Top}_{n-1}} p(v)
    \;+\sum_{v\notin \mathrm{Top}_{n-1}} \min\bigl(p(v),\tilde{q}(v)\bigr)
    \;+r(x_n)\!\sum_{v\in\mathrm{Top}_{n-1}} \max\bigl(0,\tilde{q}(v)-p(v)\bigr) \notag \\
    &\le \sum_{v\in\mathrm{Top}_{n}} p(v)
    \;+\sum_{v\notin \mathrm{Top}_{n}} \min\bigl(p(v),\tilde{q}(v)\bigr),
\end{align}
where $r=\mathrm{norm}(\max(0,p-\tilde{q}))$. In addition, $\mathcal{A}_{\text{Rheo}}=\sum_{v\in\mathrm{Top}_{n}} p(v)$ for $s=-1$.
\end{theorem}

\section{Experiments}
\label{sec:exp}
\subsection{Experimental Setup}

\textbf{Datasets and Models.}
We evaluate RheoSampling across six diverse benchmarks: Alpaca~\cite{alpaca}, GSM8K~\cite{gsm8k}, HumanEval~\cite{humaneval}, MT-bench~\cite{mt-bench}, Natural Questions~\cite{nq}, and CNN/DailyMail~\cite{cnndm}. Each dataset consists of 80 questions, spanning instruction following, mathematical reasoning, code generation, multi-turn dialogue, question answering, and summarization, ensuring a comprehensive assessment of generation quality and efficiency. For target models, we employ Llama-3.1-8B-Instruct~\cite{llama3}, Vicuna-13B-v1.3~\cite{mt-bench}, and DeepSeek-R1-Distill-Llama-8B~\cite{dsl}. All draft models use the officially released EAGLE-3 checkpoints~\cite{li2025eagle3} without further fine-tuning.

\textbf{Metrics and Implementation.}
We report two primary metrics: \textit{average acceptance length} ($\tau$), defined as the average number of tokens accepted per drafting-verification cycle, and \textit{actual speedup}, measured as the wall-clock time reduction relative to autoregressive decoding. Our implementation builds upon the open source EAGLE codebase~\cite{LiEAGLE, LiEAGLE2, li2025eagle3}. The default size of the decoding tree is 60 and the draft depth is 8, following the original EAGLE-3 setting.  All experiments are conducted on a single NVIDIA A6000 GPU over three independent runs with different random seeds. Unless otherwise specified, we adopt the default temperature $T=1.0$ and the rheostat parameter $m=1$.

\subsection{Main Results}

Table~\ref{tab:main_results} presents the end-to-end evaluation of RheoSampling against the vanilla Top-$K$ baseline across six benchmarks and three target models. RheoSampling achieves consistent improvements in mean acceptance length ($\tau$) across all configurations, with mean gains (averaged over the six benchmarks) ranging from $+0.14$ (Vicuna-13B) to $+0.22$ (DeepSeek-R1-Distill-8B). The absolute improvement varies by task and model, reflecting differences in draft-target alignment and distribution tail mass. Nevertheless, the average gains over pure top-$K$ sampling are consistently observed across all three random seeds.

These improvements in acceptance length translate directly into wall-clock speedup, though the relative speedup improvement is slightly smaller than the $\tau$ gain. For instance, on Llama-3.1-8B, $\tau$ improves by 4.2\% ($5.04 \to 5.25$) while speedup increases by 3.2\% ($2.84\times \to 2.93\times$). This stems from the additional overhead of sampling and verifying the stochastic probe, which is not present in the pure Top-$K$ baseline. Nevertheless, the net speedup is positive across all models, confirming that the theoretical benefits of stochastic exploration outweigh its marginal computational cost.

\begin{table}[t]
\centering
\small
\setlength{\tabcolsep}{5pt}
\caption{Main results with EAGLE-3 on six different tasks. We report the mean acceptance length (mean{\tiny ±std} over 3 runs) and the end-to-end speedup. V-13B, L31-8B and DSL-8B denote Vicuna-13B-v1.3, Llama-3.1-8B-Instruct and DeepSeek-R1-Distill-Llama-8B, respectively.}
\label{tab:main_results}
\begin{tabular}{@{}llcccccc cc@{}}
\toprule
\multirow{2}{*}{\textbf{Model}} & \multirow{2}{*}{\textbf{Sampling}} & \multicolumn{6}{c}{\textbf{Tasks}} & \multicolumn{2}{c}{\textbf{Average}} \\
\cmidrule(lr){3-8} \cmidrule(lr){9-10}
 & & \textbf{Alpaca} & \textbf{Math} & \textbf{Code} & \textbf{MT} & \textbf{QA} & \textbf{Sum} & \textbf{$\tau$} & \textbf{Speedup} \\
\midrule
\multirow{2}{*}{V-13B} 
 & Top-$K$ & 5.59{\tiny ±0.09} & 5.85{\tiny ±0.09} & 6.66{\tiny ±0.13} & 5.69{\tiny ±0.06} & 4.91{\tiny ±0.08} & 5.80{\tiny ±0.03} & 5.75{\tiny ±0.03} & 3.37$\times$ \\
 & Rheo & \textbf{5.92}{\tiny ±0.05} & \textbf{5.87}{\tiny ±0.05} & \textbf{6.72}{\tiny ±0.06} & \textbf{5.88}{\tiny ±0.09} & \textbf{5.00}{\tiny ±0.04} & \textbf{5.94}{\tiny ±0.07} & \textbf{5.89}{\tiny ±0.03} & \textbf{3.43}$\times$ \\
\midrule
\multirow{2}{*}{L31-8B} 
 & Top-$K$ & 5.68{\tiny ±0.02} & 5.50{\tiny ±0.03} & 6.11{\tiny ±0.09} & 4.60{\tiny ±0.06} & 3.89{\tiny ±0.04} & 4.49{\tiny ±0.07} & 5.04{\tiny ±0.02} & 2.84$\times$ \\
 & Rheo & \textbf{6.03}{\tiny ±0.12} & \textbf{5.60}{\tiny ±0.08} & \textbf{6.23}{\tiny ±0.03} & \textbf{4.84}{\tiny ±0.14} & \textbf{4.14}{\tiny ±0.04} & \textbf{4.66}{\tiny ±0.07} & \textbf{5.25}{\tiny ±0.03} & \textbf{2.93}$\times$ \\
\midrule
\multirow{2}{*}{DSL-8B} 
 & Top-$K$ & 4.48{\tiny ±0.04} & 6.70{\tiny ±0.05} & 5.43{\tiny ±0.09} & 4.93{\tiny ±0.02} & 4.08{\tiny ±0.05} & 4.32{\tiny ±0.06} & 4.99{\tiny ±0.01} & 2.89$\times$ \\
 & Rheo & \textbf{4.79}{\tiny ±0.06} & \textbf{6.71}{\tiny ±0.07} & \textbf{5.79}{\tiny ±0.05} & \textbf{5.16}{\tiny ±0.05} & \textbf{4.35}{\tiny ±0.05} & \textbf{4.47}{\tiny ±0.06} & \textbf{5.21}{\tiny ±0.02} & \textbf{2.99}$\times$ \\
\bottomrule
\end{tabular}
\end{table}

\subsection{Ablation Study on Rheostat Parameter}

As shown in Figure~\ref{fig:ablation_m}, $m=1$ achieves the best or near-best performance in most settings. $m=0$ and $m=2$ remain viable and outperform the Top-$K$ baseline, but $m \geq 3$ suffers from diminishing proxy probability, so the sampled token is pruned during reranking, yielding marginal gains.

\begin{figure}[h]
\centering
\includegraphics[width=0.9\linewidth]{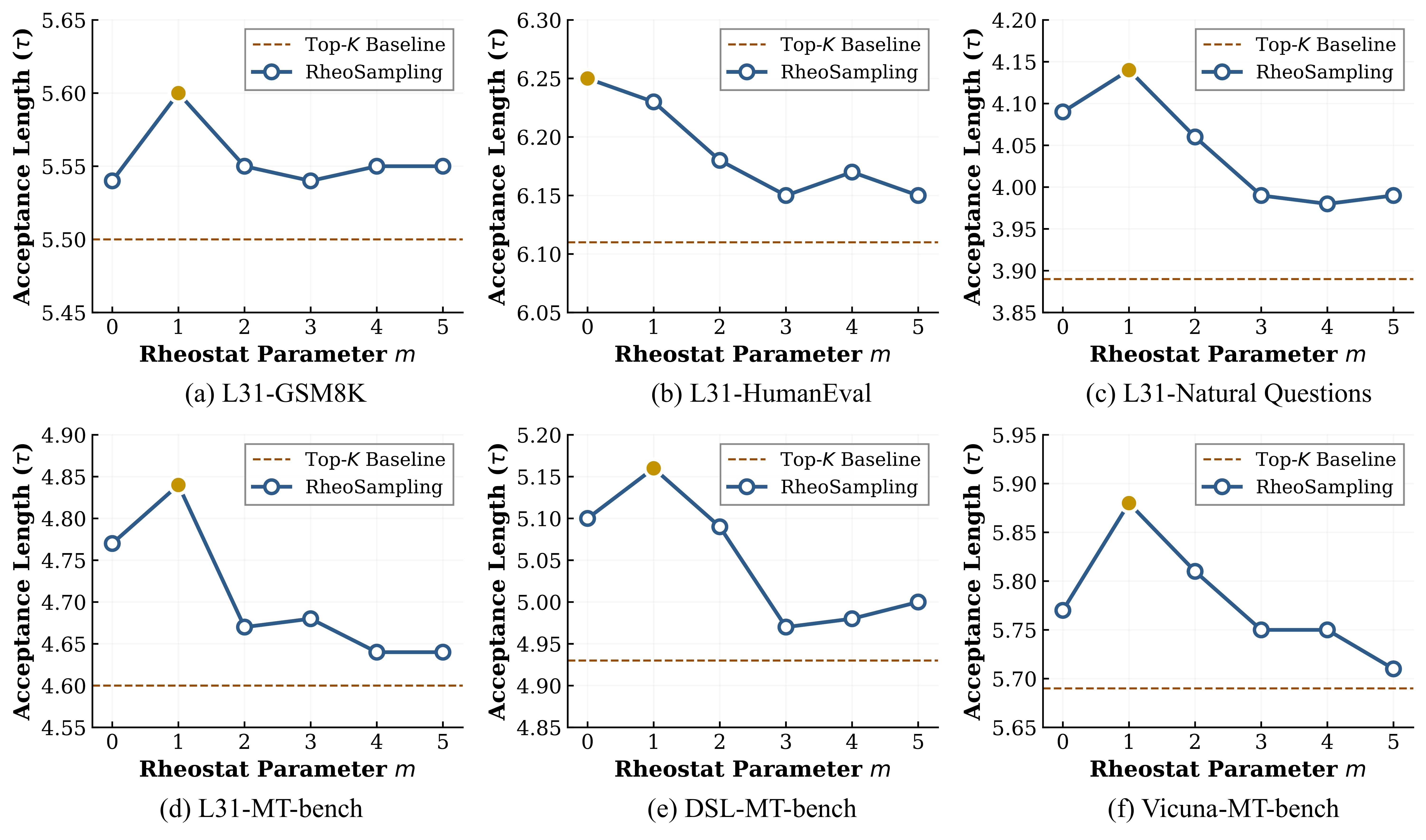}
\caption{Ablation on the rheostat parameter $m$. \textbf{Top row:} Llama-3.1-8B-Instruct across tasks (GSM8K, HumanEval, Natural Questions). \textbf{Bottom row:} Results on MT-bench across different target models (Llama-3.1-8B-Instruct, DeepSeek-R1-Distill-Llama-8B, Vicuna-13B-v1.3).} %
\label{fig:ablation_m}
\end{figure}

\textbf{Verification perspective.}
When $m=0$, the residual distribution $\tilde{q}$ degenerates to the full draft distribution $q$, reducing the verification of the sampled token to standard rejection sampling. As a result, the OT-based probability reallocation, which benefits from the exclusion of top-$m$ mass to raise the normalized residual probability, is not fully exploited. As $m$ increases, the top-$m$ mass enters the normalization denominator, enabling higher theoretical acceptance bounds under the OT framework; however, excessively large $m$ suppresses the proxy probability below the survival threshold, preventing the sampled token from realizing these gains in the final tree.

\textbf{Tree structure perspective.}
The impact of the sampled token on tree topology depends on its assigned slot. Under $m=0$, the proxy equals $q(x_1)$ with ties broken in favor of the sampled slot, anchoring the token at the first slot. Therefore, if the sampled token coincides with the original top-1 candidate, the behavior of tree expansion and pruning at this node is effectively identical to pure Top-$K$; otherwise, the tree will alter, introducing unpredictable structural deviation. In this sense, $m=0$ is an exception: it may preserve the original structure by chance, or perturb it significantly. By contrast, for any $m>0$, the proxy injection definitely affects the tree structure, though the impact becomes increasingly modest as $m$ grows and more deterministic tokens precede the sampled one.

\textbf{The sweet spot of rheostat.}
Under standard stochastic decoding conditions, $m=1$ strikes a favorable balance: it anchors the sampled token at the second slot, ensuring sufficient proxy mass to survive reranking while fully leveraging the OT-based verification advantage. The structural impact is moderate and stable, unlike the all-or-nothing behavior of $m=0$. In fact, this trade-off is temperature-dependent. We analyze this interaction in Section~\ref{subsec:temperature}.

\subsection{Impact of Temperature}
\label{subsec:temperature}

\textbf{Decoupling tree construction from temperature.}
An important implementation detail is that tree construction (expansion and reranking) operates on the \textit{original} draft probabilities without temperature scaling, consistent with standard EAGLE practice. Temperature only affects the sampling and verification stages. This decoupling is necessary because low temperatures can distort the relative ranking of tokens and destroy the ordinal information.

\textbf{Acceptance Overview.} Figure~\ref{fig:ablation_t} demonstrates the impact of temperature. As temperature decreases, the distribution becomes increasingly sharp, causing any probability-based sampling to degenerate toward deterministic top-token selection. Mathematically, the verification strategies such as Top-$K$, RRS-based or our OT-based scheme become nearly equivalent in this limit, and their difference in acceptance rate narrows. All methods exhibit rising acceptance lengths as $T$ drops, though their respective structural behaviors remain distinct.

\textbf{Degeneracy asymmetry.}
When $T \to 0$, RheoSampling degenerates along two axes: \textit{sampling degeneracy} (the stochastic probe converges to deterministic selection) and \textit{structural degeneracy} (the tree topology matches pure Top-$K$). Yet $m=0$ and $m=1$ degenerate differently. Under $m=0$, the proxy is $q(x_1)+\epsilon$. At low temperature, the residual concentrates on the true top-1 candidate, which is almost surely sampled. Since the proxy closely matches the original top-1 probability used for tree construction, the resulting topology is identical to pure Top-$K$. Thus, $m=0$ achieves \textit{dual degeneracy}: both sampling and structure collapse to the Top-$K$ baseline. Under $m=1$, only sampling degenerates. At $T \to 0$, the residual concentrates on the true second-ranked token, but its proxy $\min\{q(x_1), z\}$ does not equal its original probability $q(x_2)$. While verification evaluates the token using its true sampling probability, tree construction still operates on a mismatched proxy, perturbing path scores. Consequently, $m=1$ exhibits only \textit{single degeneracy}: sampling collapses to Top-$K$, but the tree structure remains distinct. This structural discrepancy causes the acceptance gap between $m=1$ and the Top-$K$ / $m=0$ baselines as $T \to 0$.

\begin{figure}[t]
\centering
\includegraphics[width=1.0\linewidth]{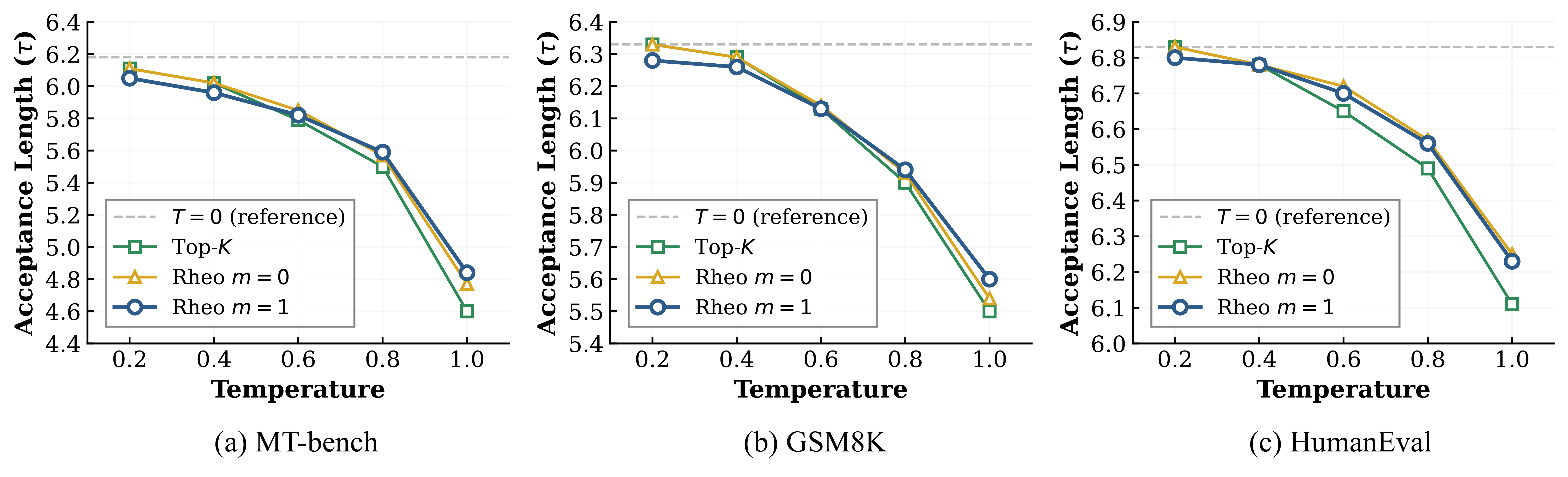}
\caption{Temperature ablation on Llama-3.1-8B-Instruct across MT-bench, GSM8K, and HumanEval. The dashed grey line marks the $T=0$ (greedy decoding) reference.}
\label{fig:ablation_t}
\end{figure}

\textbf{Remarks on proxy design.}
As discussed in Section~\ref{sec:theo} and Lemma~\ref{lem:local-rank}, any proxy probability that does not depend on the realized sampled token and is no smaller than $q(x_{m+1})$ (with ties broken in favor of the sampled slot) preserves the losslessness guarantee. Setting the proxy to $q(x_{m+1})+\epsilon$ would achieve the ``dual degeneracy'' for arbitrary $m$ at low temperature, since the sampled token would carry its true original ranking probability into tree construction. Conversely, our choice of $\min\{q(x_m), z\}$ prioritizes the survival probability of the sampled token during reranking by granting it a larger proxy mass. \textit{Both proxy designs are valid instantiations within the RheoSampling framework; the choice of proxy is a matter of empirical tuning, not theoretical necessity.}

\section{Related Work}

\paragraph{Speculative decoding and tree structures.}
Speculative decoding introduces a drafting-verification paradigm that accelerates LLM inference without sacrificing generation quality~\cite{LeviathanSpec, ChenSpec}. Early works primarily employed static tree structures for parallel verification, such as the manually designed trees in Medusa~\cite{Medusa} and SpecInfer~\cite{SpecInfer}, as well as EAGLE-1~\cite{LiEAGLE}. Subsequent research shifted towards dynamic, context-aware tree construction. EAGLE-2~\cite{LiEAGLE2} and EAGLE-3~\cite{li2025eagle3} employ an expand-then-rerank paradigm that adaptively selects candidates based on path probabilities. Opt-Tree~\cite{optree} shares the motivation for context-aware dynamic construction, employing a slightly different optimization-based approach. Sequoia~\cite{ChenSequoia} adopts a hardware-aware tree topology via dynamic programming.

\paragraph{Multi-draft verification strategies.}
Verification strategies have evolved from chain-based to tree-based settings. SpecInfer~\cite{SpecInfer} introduced Recursive Rejection Sampling (RRS) for multi-draft scenarios, later refined into RRS without replacement \cite{LiEAGLE, ChenSequoia, YangMCSD, JeonRSD} to prevent repeat sampling of identical tokens. From an optimal transport (OT) perspective, SpecTr~\cite{SpecTr} first formulated multi-draft verification as OT problem. Subsequent works such as SpecHub~\cite{SpecHub} and Greedy method~\cite{TowardsOptimal} explored hybrid drafting strategies that combine deterministically selected drafts with sampled tokens. However, these approaches are designed for static candidate pools, whereas RheoSampling integrates hybrid stochasticity into \textit{dynamic} tree construction, enabling context-aware expansion with lossless OT-based verification.

\section{Conclusion}

We present \textbf{RheoSampling}, resolving the one-hot dilemma in dynamic-tree speculative decoding where deterministic expansion inherently conflicts with stochastic generation. By assigning a sampled token a \textit{proxy probability} for context-aware tree construction and its \textit{true sampling probability} for OT-based verification, our dual-identity framework successfully decouples topology from randomness. Crucially, we establish the first rigorous proof of losslessness for stochastic dynamic trees through a novel equivalence-class analysis that compresses the otherwise intractable probability space. Coupled with a sparse draft mechanism, RheoSampling translates these theoretical guarantees into consistent empirical speedups and superior acceptance rates over state-of-the-art baselines. By unifying dynamic topologies with mathematically sound stochastic sampling, this work provides a foundational template for adapting advanced verification algorithms to complex tree structures.

\begin{ack}
We thank the Academy of Mathematics and Systems Science, Chinese Academy of Sciences, The University of Tokyo, and Lenovo AI Technology Center for their support of this work.
\end{ack}

\bibliographystyle{acl_natbib}  %
\small
\bibliography{Reference}

@article{GPT-4,
  author={OpenAI},
  title={{GPT-4} Technical Report},
  journal={arXiv preprint arXiv:2303.08774},
  year={2023},
}

@article{GPT-5,
  author       = {OpenAI},
  title        = {OpenAI {GPT-5} System Card},
  journal={arXiv preprint arXiv:2601.03267},
  year         = {2026},
}

@article{llama3,
  title={The Llama 3 Herd of Models}, 
  author={Aaron Grattafiori and Abhimanyu Dubey and Abhinav Jauhri and Abhinav Pandey and Abhishek Kadian and others},
  journal={arXiv preprint arXiv:2407.21783}, 
  year={2024},
}

@article{qwen3,
    title={Qwen3 Technical Report}, 
    author={An Yang and Anfeng Li and Baosong Yang and Beichen Zhang and Binyuan Hui and Bo Zheng and Bowen Yu and Chang Gao and Chengen Huang and Chenxu Lv and Chujie Zheng and Dayiheng Liu and Fan Zhou and Fei Huang and Feng Hu and Hao Ge and Haoran Wei and Huan Lin and Jialong Tang and Jian Yang and Jianhong Tu and Jianwei Zhang and Jianxin Yang and Jiaxi Yang and Jing Zhou and Jingren Zhou and Junyang Lin and Kai Dang and Keqin Bao and Kexin Yang and Le Yu and Lianghao Deng and Mei Li and Mingfeng Xue and Mingze Li and Pei Zhang and Peng Wang and Qin Zhu and Rui Men and Ruize Gao and Shixuan Liu and Shuang Luo and Tianhao Li and Tianyi Tang and Wenbiao Yin and Xingzhang Ren and Xinyu Wang and Xinyu Zhang and Xuancheng Ren and Yang Fan and Yang Su and Yichang Zhang and Yinger Zhang and Yu Wan and Yuqiong Liu and Zekun Wang and Zeyu Cui and Zhenru Zhang and Zhipeng Zhou and Zihan Qiu},
    journal = {arXiv preprint arXiv:2505.09388},
    year={2025}
}

@inproceedings{LeviathanSpec,
  author       = {Yaniv Leviathan and
                  Matan Kalman and
                  Yossi Matias},
  title        = {Fast Inference from Transformers via Speculative Decoding},
  booktitle    = {International Conference on Machine Learning, {ICML} 2023, 23-29 July
                  2023, Honolulu, Hawaii, {USA}},
  year         = {2023},
}

@article{ChenSpec,
  author       = {Charlie Chen and
                  Sebastian Borgeaud and
                  Geoffrey Irving and
                  Jean{-}Baptiste Lespiau and
                  Laurent Sifre and
                  John Jumper},
  title        = {Accelerating Large Language Model Decoding with Speculative Sampling},
    journal={arXiv preprint arXiv:2302.01318},
    year = {2023}
    }

@inproceedings{SpecInfer,
  author       = {Xupeng Miao and
                  Gabriele Oliaro and
                  Zhihao Zhang and
                  Xinhao Cheng and
                  Zeyu Wang and
                  Zhengxin Zhang and
                  Rae Ying Yee Wong and
                  Alan Zhu and
                  Lijie Yang and
                  Xiaoxiang Shi and
                  Chunan Shi and
                  Zhuoming Chen and
                  Daiyaan Arfeen and
                  Reyna Abhyankar and
                  Zhihao Jia},
  editor       = {Rajiv Gupta and
                  Nael B. Abu{-}Ghazaleh and
                  Madan Musuvathi and
                  Dan Tsafrir},
  title        = {SpecInfer: Accelerating Large Language Model Serving with Tree-based
                  Speculative Inference and Verification},
  booktitle    = {Proceedings of the 29th {ACM} International Conference on Architectural
                  Support for Programming Languages and Operating Systems, Volume 3,
                  {ASPLOS} 2024, La Jolla, CA, USA, 27 April 2024- 1 May 2024},
  year         = {2024},
}

@inproceedings{ChenSequoia,
  author       = {Zhuoming Chen and
                  Avner May and
                  Ruslan Svirschevski and
                  Yuhsun Huang and
                  Max Ryabinin and
                  Zhihao Jia and
                  Beidi Chen},
  title        = {Sequoia: Scalable and Robust Speculative Decoding},
  booktitle    = {Advances in Neural Information Processing Systems 37: Annual Conference
                  on Neural Information Processing Systems 2024, NeurIPS 2024, Vancouver,
                  BC, Canada, December 10 - 15, 2024},
  year         = {2024},
}

@inproceedings{LiEAGLE,
  author       = {Yuhui Li and
                  Fangyun Wei and
                  Chao Zhang and
                  Hongyang Zhang},
  title        = {{EAGLE:} Speculative Sampling Requires Rethinking Feature Uncertainty},
  booktitle    = {Forty-first International Conference on Machine Learning, {ICML} 2024,
                  Vienna, Austria, July 21-27, 2024},
  year         = {2024},
}

@article{YangMCSD,
  author       = {Sen Yang and
                  Shujian Huang and
                  Xinyu Dai and
                  Jiajun Chen},
  title        = {Multi-Candidate Speculative Decoding},
  journal      = {arXiv preprint arXiv:2401.06706},
  year         = {2024},
}

@article{JeonRSD,
  author       = {Wonseok Jeon and
                  Mukul Gagrani and
                  Raghavv Goel and
                  Junyoung Park and
                  Mingu Lee and
                  Christopher Lott},
  title        = {Recursive Speculative Decoding: Accelerating {LLM} Inference via Sampling
                  Without Replacement},
  journal      = {arXiv preprint arXiv:2402.14160},
  year         = {2024},
}

@inproceedings{SpecTr,
  author       = {Ziteng Sun and
                  Ananda Theertha Suresh and
                  Jae Hun Ro and
                  Ahmad Beirami and
                  Himanshu Jain and
                  Felix X. Yu},
  title        = {SpecTr: Fast Speculative Decoding via Optimal Transport},
  booktitle    = {Advances in Neural Information Processing Systems 36: Annual Conference
                  on Neural Information Processing Systems 2023, NeurIPS 2023, New Orleans,
                  LA, USA, December 10 - 16, 2023},
  year         = {2023},
}

@inproceedings{SpecHub,
  author       = {Ryan Sun and
                  Tianyi Zhou and
                  Xun Chen and
                  Lichao Sun},
  title        = {SpecHub: Provable Acceleration to Multi-Draft Speculative Decoding},
  booktitle    = {Proceedings of the 2024 Conference on Empirical Methods in Natural
                  Language Processing, {EMNLP} 2024, Miami, FL, USA, November 12-16,
                  2024},
  year         = {2024},
}

@inproceedings{TowardsOptimal,
  author       = {Zhengmian Hu and
                  Tong Zheng and
                  Vignesh Viswanathan and
                  Ziyi Chen and
                  Ryan A. Rossi and
                  Yihan Wu and
                  Dinesh Manocha and
                  Heng Huang},
  title        = {Towards Optimal Multi-draft Speculative Decoding},
  booktitle={Proceedings of the International Conference on Learning Representations},
  year         = {2025},
}

@inproceedings{Medusa,
  author={Tianle Cai and Yuhong Li and Zhengyang Geng and Hongwu Peng and Jason D. Lee and Deming Chen and Tri Dao},
  title={Medusa: Simple {LLM} Inference Acceleration Framework with Multiple Decoding Heads},
  booktitle={Proceedings of the International Conference on Machine Learning},
  year={2024},
}

@inproceedings{LiEAGLE2,
  author       = {Yuhui Li and
                  Fangyun Wei and
                  Chao Zhang and
                  Hongyang Zhang},
  title        = {{EAGLE-2:} Faster Inference of Language Models with Dynamic Draft
                  Trees},
  booktitle    = {Proceedings of the 2024 Conference on Empirical Methods in Natural
                  Language Processing, {EMNLP} 2024, Miami, FL, USA, November 12-16,
                  2024},
  year         = {2024},
}

@article{optree,
  author       = {Jikai Wang and
                  Yi Su and
                  Juntao Li and
                  Qingrong Xia and
                  Zi Ye and
                  Xinyu Duan and
                  Zhefeng Wang and
                  Min Zhang},
  title        = {OPT-Tree: Speculative Decoding with Adaptive Draft Tree Structure},
  journal      = {Trans. Assoc. Comput. Linguistics},
  year         = {2025},
}

@inproceedings{mt-bench,
  author       = {Lianmin Zheng and
                  Wei{-}Lin Chiang and
                  Ying Sheng and
                  Siyuan Zhuang and
                  Zhanghao Wu and
                  Yonghao Zhuang and
                  Zi Lin and
                  Zhuohan Li and
                  Dacheng Li and
                  Eric P. Xing and
                  Hao Zhang and
                  Joseph E. Gonzalez and
                  Ion Stoica},
  title        = {Judging LLM-as-a-Judge with MT-Bench and Chatbot Arena},
  booktitle    = {Advances in Neural Information Processing Systems 36: Annual Conference
                  on Neural Information Processing Systems 2023, NeurIPS 2023, New Orleans,
                  LA, USA, December 10 - 16, 2023},
  year         = {2023},
}

@inproceedings{cnndm,
  author       = {Ramesh Nallapati and
                  Bowen Zhou and
                  C{\'{\i}}cero Nogueira dos Santos and
                  {\c{C}}aglar G{\"{u}}l{\c{c}}ehre and
                  Bing Xiang},
  title        = {Abstractive Text Summarization using Sequence-to-sequence RNNs and
                  Beyond},
  booktitle    = {Proceedings of the 20th {SIGNLL} Conference on Computational Natural
                  Language Learning, CoNLL 2016, Berlin, Germany, August 11-12, 2016},
  year         = {2016},

}

@article{gsm8k,
  author       = {Karl Cobbe and
                  Vineet Kosaraju and
                  Mohammad Bavarian and
                  Mark Chen and
                  Heewoo Jun and
                  Lukasz Kaiser and
                  Matthias Plappert and
                  Jerry Tworek and
                  Jacob Hilton and
                  Reiichiro Nakano and
                  Christopher Hesse and
                  John Schulman},
  title        = {Training Verifiers to Solve Math Word Problems},
  journal      = {arXiv preprint arXiv:2110.14168},
  year         = {2021},
}

@inproceedings{li2025eagle3,
    author = {Yuhui Li and Fangyun Wei and Chao Zhang and Hongyang Zhang},
    title = {{EAGLE-3}: Scaling up Inference Acceleration of Large Language Models via Training-Time Test}, 
    booktitle = {Annual Conference on Neural Information Processing Systems},
    year = {2025}
}

@inproceedings{
thomas2026global,
title={Global Resolution: Optimal Multi-Draft Speculative Sampling via Convex Optimization},
author={Rahul Krishna Thomas and Arka Pal},
booktitle={The Fourteenth International Conference on Learning Representations},
year={2026},
}

@article{humaneval,
  title={Evaluating Large Language Models Trained on Code},
  author={Mark Chen and Jerry Tworek and Heewoo Jun and Qiming Yuan and Henrique Ponde de Oliveira Pinto and Jared Kaplan and Harri Edwards and Yuri Burda and Nicholas Joseph and Greg Brockman and Alex Ray and Raul Puri and Gretchen Krueger and Michael Petrov and Heidy Khlaaf and Girish Sastry and Pamela Mishkin and Brooke Chan and Scott Gray and Nick Ryder and Mikhail Pavlov and Alethea Power and Lukasz Kaiser and Mohammad Bavarian and Clemens Winter and Philippe Tillet and Felipe Petroski Such and Dave Cummings and Matthias Plappert and Fotios Chantzis and Elizabeth Barnes and Ariel Herbert-Voss and William Hebgen Guss and Alex Nichol and Alex Paino and Nikolas Tezak and Jie Tang and Igor Babuschkin and Suchir Balaji and Shantanu Jain and William Saunders and Christopher Hesse and Andrew N. Carr and Jan Leike and Josh Achiam and Vedant Misra and Evan Morikawa and Alec Radford and Matthew Knight and Miles Brundage and Mira Murati and Katie Mayer and Peter Welinder and Bob McGrew and Dario Amodei and Sam McCandlish and Ilya Sutskever and Wojciech Zaremba},
  year={2021},
  journal={arXiv preprint arXiv:2107.03374},
}

@misc{alpaca,
  author = {Rohan Taori and Ishaan Gulrajani and Tianyi Zhang and Yann Dubois and Xuechen Li and Carlos Guestrin and Percy Liang and Tatsunori B. Hashimoto },
  title = {Stanford Alpaca: An Instruction-following LLaMA model},
  year = {2023},
  journal = {GitHub repository},
  howpublished = {\url{https://github.com/tatsu-lab/stanford_alpaca}},
}

@article{nq,
    title = "Natural Questions: A Benchmark for Question Answering Research",
    author = "Kwiatkowski, Tom  and
      Palomaki, Jennimaria  and
      Redfield, Olivia  and
      Collins, Michael  and
      Parikh, Ankur  and
      Alberti, Chris  and
      Epstein, Danielle  and
      Polosukhin, Illia  and
      Devlin, Jacob  and
      Lee, Kenton  and
      Toutanova, Kristina  and
      Jones, Llion  and
      Kelcey, Matthew  and
      Chang, Ming-Wei  and
      Dai, Andrew M.  and
      Uszkoreit, Jakob  and
      Le, Quoc  and
      Petrov, Slav",
    editor = "Lee, Lillian  and
      Johnson, Mark  and
      Roark, Brian  and
      Nenkova, Ani",
    journal = "Transactions of the Association for Computational Linguistics",
    volume = "7",
    year = "2019",
    pages = "452--466",
}

@article{dsl,
   title={DeepSeek-R1 incentivizes reasoning in LLMs through reinforcement learning},
   volume={645},
   ISSN={1476-4687},
   number={8081},
   journal={Nature},
   publisher={Springer Science and Business Media LLC},
   author={Guo, Daya and Yang, Dejian and Zhang, Haowei and Song, Junxiao and Wang, Peiyi and Zhu, Qihao and Xu, Runxin and Zhang, Ruoyu and Ma, Shirong and Bi, Xiao and Zhang, Xiaokang and Yu, Xingkai and Wu, Yu and Wu, Z. F. and Gou, Zhibin and Shao, Zhihong and Li, Zhuoshu and Gao, Ziyi and Liu, Aixin and Xue, Bing and Wang, Bingxuan and Wu, Bochao and Feng, Bei and Lu, Chengda and Zhao, Chenggang and Deng, Chengqi and Ruan, Chong and Dai, Damai and Chen, Deli and Ji, Dongjie and Li, Erhang and Lin, Fangyun and Dai, Fucong and Luo, Fuli and Hao, Guangbo and Chen, Guanting and Li, Guowei and Zhang, H. and Xu, Hanwei and Ding, Honghui and Gao, Huazuo and Qu, Hui and Li, Hui and Guo, Jianzhong and Li, Jiashi and Chen, Jingchang and Yuan, Jingyang and Tu, Jinhao and Qiu, Junjie and Li, Junlong and Cai, J. L. and Ni, Jiaqi and Liang, Jian and Chen, Jin and Dong, Kai and Hu, Kai and You, Kaichao and Gao, Kaige and Guan, Kang and Huang, Kexin and Yu, Kuai and Wang, Lean and Zhang, Lecong and Zhao, Liang and Wang, Litong and Zhang, Liyue and Xu, Lei and Xia, Leyi and Zhang, Mingchuan and Zhang, Minghua and Tang, Minghui and Zhou, Mingxu and Li, Meng and Wang, Miaojun and Li, Mingming and Tian, Ning and Huang, Panpan and Zhang, Peng and Wang, Qiancheng and Chen, Qinyu and Du, Qiushi and Ge, Ruiqi and Zhang, Ruisong and Pan, Ruizhe and Wang, Runji and Chen, R. J. and Jin, R. L. and Chen, Ruyi and Lu, Shanghao and Zhou, Shangyan and Chen, Shanhuang and Ye, Shengfeng and Wang, Shiyu and Yu, Shuiping and Zhou, Shunfeng and Pan, Shuting and Li, S. S. and Zhou, Shuang and Wu, Shaoqing and Yun, Tao and Pei, Tian and Sun, Tianyu and Wang, T. and Zeng, Wangding and Liu, Wen and Liang, Wenfeng and Gao, Wenjun and Yu, Wenqin and Zhang, Wentao and Xiao, W. L. and An, Wei and Liu, Xiaodong and Wang, Xiaohan and Chen, Xiaokang and Nie, Xiaotao and Cheng, Xin and Liu, Xin and Xie, Xin and Liu, Xingchao and Yang, Xinyu and Li, Xinyuan and Su, Xuecheng and Lin, Xuheng and Li, X. Q. and Jin, Xiangyue and Shen, Xiaojin and Chen, Xiaosha and Sun, Xiaowen and Wang, Xiaoxiang and Song, Xinnan and Zhou, Xinyi and Wang, Xianzu and Shan, Xinxia and Li, Y. K. and Wang, Y. Q. and Wei, Y. X. and Zhang, Yang and Xu, Yanhong and Li, Yao and Zhao, Yao and Sun, Yaofeng and Wang, Yaohui and Yu, Yi and Zhang, Yichao and Shi, Yifan and Xiong, Yiliang and He, Ying and Piao, Yishi and Wang, Yisong and Tan, Yixuan and Ma, Yiyang and Liu, Yiyuan and Guo, Yongqiang and Ou, Yuan and Wang, Yuduan and Gong, Yue and Zou, Yuheng and He, Yujia and Xiong, Yunfan and Luo, Yuxiang and You, Yuxiang and Liu, Yuxuan and Zhou, Yuyang and Zhu, Y. X. and Huang, Yanping and Li, Yaohui and Zheng, Yi and Zhu, Yuchen and Ma, Yunxian and Tang, Ying and Zha, Yukun and Yan, Yuting and Ren, Z. Z. and Ren, Zehui and Sha, Zhangli and Fu, Zhe and Xu, Zhean and Xie, Zhenda and Zhang, Zhengyan and Hao, Zhewen and Ma, Zhicheng and Yan, Zhigang and Wu, Zhiyu and Gu, Zihui and Zhu, Zijia and Liu, Zijun and Li, Zilin and Xie, Ziwei and Song, Ziyang and Pan, Zizheng and Huang, Zhen and Xu, Zhipeng and Zhang, Zhongyu and Zhang, Zhen},
   year={2025},
   pages={633–638} }
\normalsize


\appendix

\section{Comparison on Verification Strategies}
\label{appendix:rrs}

While RheoSampling's hybrid draft strategy inherently improves acceptance by covering the distribution tail, the verification algorithm itself also contributes to the final rate. As we have mentioned in \ref{sec:verification}, we design RheoVerification from an Optimal Transport perspective, achieving a higher acceptance rate for the sampled token while maintaining losslessness. In contrast, Recursive Rejection Sampling without replacement (RRSw) processes candidates sequentially without such global optimization.

\begin{table}[h]
\centering
\small
\setlength{\tabcolsep}{3pt}
\caption{End-to-end performance comparison across benchmarks. RheoSampling combined with RheoVerification consistently outperforms both the RRSw variant and the vanilla Top-$K$ baseline.}
\label{tab:rrs_results}
\begin{tabular}{ll|cccccc|c}
\toprule
\textbf{Sampling} & \textbf{Verification} & \textbf{Alpaca} & \textbf{GSM8K} & \textbf{HumanEval} & \textbf{MT-bench} & \textbf{Natural Q.} & \textbf{CNN/DM} & \textbf{Average} \\
\midrule
Top-$K$ & Vanilla & 5.68 & 5.50 & 6.11 & 4.60 & 3.89 & 4.49 & 5.04 \\
\midrule
\multirow{2}{*}{Rheo} & RRSw & 5.88 & 5.47 & 6.18 & 4.67 & 4.09 & 4.58 & 5.15 \\
                      & Rheo & \textbf{6.03} & \textbf{5.60} & \textbf{6.23} & \textbf{4.84} & \textbf{4.14} & \textbf{4.66} & \textbf{5.25} \\
\bottomrule
\end{tabular}
\end{table}

Table~\ref{tab:rrs_results} presents the end-to-end performance across multiple benchmarks. RheoSampling paired with RRSw already outperforms the vanilla Top-$K$ baseline, confirming the benefit of injecting stochasticity into the draft tree. Replacing RRSw with RheoVerification yields further consistent gains, validating that the RheoVerification captures additional acceptance probability beyond sequential rejection.

To isolate the pure effect of the verification strategy, we conduct a controlled demo experiment on MT-bench with fixed settings: draft depth is restricted to one layer, and the three candidate slots are set to top-1, sampled token, and the highest-ranked token outside the first two. Under this configuration, the difference between the two Rheo rows stems solely from the verification algorithm. As shown in Table~\ref{tab:acceptance_demo}, RheoVerification achieves 85.4\% acceptance, compared to 83.5\% under RRSw and 80.0\% under vanilla Top-$K$. The 3.5\% gap between RRSw and Top-$K$ reflects the gain from hybrid sampling alone, while the additional 1.9\% gap between RheoVerification and RRSw precisely reflects the gain from OT-based mass allocation.

\begin{table}[h]
\centering
\caption{Acceptance rate comparison under controlled settings (MT-bench, draft depth=1, fixed candidate slots: top-1, sampled-1, and top-1 of remaining).}
\label{tab:acceptance_demo}
\begin{tabular}{llc}
\toprule
\textbf{Sampling} & \textbf{Verification} & \textbf{Acceptance} \\
\midrule
Top-$K$ & Vanilla & 80.0\% \\
\midrule
\multirow{2}{*}{Rheo} & RRSw & 83.5\% \\
                      & Rheo & \textbf{85.4\%} \\
\bottomrule
\end{tabular}
\end{table}

\section{Discussion on Sparse Draft Probability}
\label{appendix:sparse}

To validate the practical impact of sparse draft distributions, we conduct ablation studies on MT-bench, measuring draft coverage, acceptance length ($\tau$), and drafting latency across varying support sizes. Results are summarized in Table~\ref{tab:sparse_support}.

\paragraph{Impact on acceptance rate.}
As shown in Table~\ref{tab:sparse_support}, the primary purpose of the \textit{Full} configuration is to provide an empirical reference value (4.89) using the dense draft distribution. Sparse support sizes from 128 to 1024 achieve comparable $\tau$ values (4.78--4.85), confirming that truncation does not systematically degrade acceptance performance. This is expected for two reasons: first, even at support size 128, the coverage of the original draft distribution exceeds $91\%$ on MT-bench (even higher on other datasets such as $\sim$96\% on Alpaca and $\sim$95\% on HumanEval), meaning the sparse and dense distributions are nearly identical for practical purposes. Second, the sparse operation only affects the single stochastically sampled token, leaving all deterministic top-$K$ candidates untouched. The minor fluctuations in $\tau$ across support sizes are well within statistical noise, and notably, the 128-support configuration actually outperforms the 512-support variant, suggesting that sparse approximation does not introduce consistent penalty.

\paragraph{Latency.}
For support sizes below 1024, drafting latency remains stable (within $\pm$0.2 ms of the 128 baseline), indicating that sparse truncation introduces negligible overhead to the overall tree construction pipeline. We adopt 128 as the default to avoid over-engineering hyperparameters.

\begin{table}[h]
\centering
\small
\setlength{\tabcolsep}{5pt}
\caption{Effect of sparse support size on draft coverage and average acceptance length ($\tau$). The \textit{Full} row provides an empirical reference using the dense vocabulary distribution. Drafting latency is reported for completeness; the Full configuration uses native PyTorch operators such as concatenation and multinomial, and is not optimized for speed.}
\label{tab:sparse_support}
\begin{tabular}{@{}lcccc@{}}
\toprule
\textbf{Sampling} & \textbf{Support Size} & \textbf{Coverage of $\textbf{q}$} & $\boldsymbol{\tau}$ & \textbf{Drafting Latency}\\
\midrule
\multirow{5}{*}{Rheo} & 128 & 91.0\% & 4.84 & 13.1 ms\\
                      & 256 & 91.8\% & 4.85 & 13.2 ms\\
                      & 512 & 92.2\% & 4.78 & 13.3 ms\\
                      & 1024 & 94.1\% & 4.83 & 13.3 ms\\
                      & Full & 100.0\% & 4.89 & 18.8 ms$^\dagger$\\
\midrule
Top-K & N/A & N/A & 4.60 & 12.9 ms\\
\bottomrule
\end{tabular}
\\[3pt]
\footnotesize $^\dagger$PyTorch native implementation over the full vocabulary and not optimized for speed; included as a $\tau$ reference only.
\end{table}


\section{Losslessness of RheoSampling}
\label{app:proof}

\subsection{Setup and Notation}
\label{app:setup}

Let $p$ denote the target next-token distribution at a given decoding layer, conditioned on the prefix generated so far. 
Let $q$ denote the original draft distribution over the full/sparse vocabulary $\mathcal{V}$.

\paragraph{Hybrid candidate pool.} 
For a fixed parent node, the top-$m$ deterministic tokens $\{x_1,\dots,x_m\}$ with probabilities $\{q(x_i)\}_{i=1}^m$, and the residual mass
\begin{equation*}
    z \;=\; 1-\sum_{i=1}^m q(x_i).
\end{equation*}
The stochastic token $Y$ is sampled from the tail distribution
\begin{equation*}
    \tilde{q}(Y) \;=\; \frac{q(Y)}{z},\qquad \forall Y\notin\{x_1,\dots,x_m\}.
\end{equation*}
The remaining $K-m-1$ slots are filled with the highest-ranked tokens from $\mathcal{V}\setminus\{x_1,\dots,x_m,Y\}$ without replacement, denoted $\{X_{m+2},\dots,X_K\}$. 
Each node in the candidate pool carries a \emph{proxy probability} for tree construction:
\begin{itemize}
    \item Top-$m$ tokens: $q_{\mathrm{proxy}}(x_i) = q(x_i)$, $i=1,\dots,m$;
    \item Sampled token: $q_{\mathrm{proxy}}(Y)=\min\{q(x_m),z\}$;
    \item Fill tokens: $q_{\mathrm{proxy}}(X_j) = q(X_j)$, $j=m+2,\dots,K$.
\end{itemize}
\emph{Special case $m=0$.} Here $x_m$ is undefined and $z=1$. We set $q_{\mathrm{proxy}}(Y)=q(x_1)+\epsilon$ as in Section~\ref{sec:method}, where $+\epsilon$ is symbolic: numerically $q_{\mathrm{proxy}}(Y)=q(x_1)$, and any tie with a fill token is resolved in favor of the sampled slot, which occupies the first (leftmost) slot and is therefore handled by the sibling-order tie-breaking rule below.
Throughout this appendix, $\min\{q(x_m),z\}$ is read as this convention when $m=0$.

\paragraph{Path scores and global pruning.} 
For any node $u$ in the unpruned tree $\mathcal{U}$, its path proxy score is the product of proxy probabilities along the root-to-$u$ path. 
The global pruning operator $\mathrm{Prune}_R(\mathcal{U})$ retains exactly the $R$ nodes with highest path scores (ties broken arbitrarily but deterministically). 
Given $\mathcal{U}$, the pruned tree $T_R=\mathrm{Prune}_R(\mathcal{U})$ is fully deterministic.

\paragraph{Tie-breaking rule.} 
Global top-$R$ pruning ranks nodes first by path proxy score in descending order. 
When scores are equal, we break ties by depth (shallower nodes prioritized over deeper nodes) and then by left-to-right sibling order. 
Under this rule, the pruned dynamic tree is always ancestor-closed and connected.

\paragraph{Verification input.} 
At a fixed layer $\ell$ of the pruned tree $T_R$, let $\mathcal{C}=\{u_1,\dots,u_n\}$ be the set of candidate tokens at that layer, and let $s\in\{-1,1,\dots,n\}$ indicate the index of the sampled token ($s=-1$ if the sampled node was pruned away).

\subsection{Structural Properties of the Proxy Scores}
\label{app:structural}

The first lemma states that, within any single candidate pool, the sampled token is always ranked above every fill token, regardless of its realized identity.

\begin{lemma}[Local ranking preservation]
\label{lem:local-rank}
The proxy probability of the sampled token satisfies
\begin{equation*}
    q_{\mathrm{proxy}}(Y) \;=\; \min\{q(x_m),\,z\} \;\ge\; q(x_{m+1}),
    \label{eq:proxy-lower-bound}
\end{equation*}
where $q(x_{m+1})$ is the $(m+1)$-th largest probability in the draft distribution $q$. 
Consequently, the proxy probability of every fill token $X_j$ satisfies $q(X_j)\le q_{\mathrm{proxy}}(Y)$.
\end{lemma}

\begin{proof}
By sorting, $q(x_{m+1})\le q(x_m)$. 
Moreover $z=\sum_{i=m+1}^{|\mathcal{V}|}q(x_i)\ge q(x_{m+1})$. 
Hence $q(x_{m+1})\le\min\{q(x_m),z\}=q_{\mathrm{proxy}}(Y)$.
For $m=0$, $q_{\mathrm{proxy}}(Y)=q(x_1)$ numerically, and the bound holds with equality.
Since every fill token is drawn from $\mathcal{V}\setminus\{x_1,\dots,x_m,Y\}$, its proxy probability is at most $q(x_{m+1})$.
\end{proof}

Lemma~\ref{lem:local-rank} guarantees that, although the \emph{identities} of the fill tokens $X_j$ depend on the realized $Y$ (due to sampling without replacement), their \emph{scores} can never surpass the sampled token's proxy score. 
This fixes the sampled node at rank $m+1$ inside its sibling group.

The second lemma captures the ancestor-closed nature of global pruning.

\begin{lemma}[Ancestor monotonicity]
\label{lem:ancestor}
For any node $u$ in an unpruned dynamic tree $\mathcal{U}$, let $S(u)$ denote its path proxy score (the product of proxy probabilities along the root-to-$u$ path). 
If $v$ is a child of $u$, then $S(v)\le S(u)$.
\end{lemma}

\begin{proof}
$S(v)=S(u)\cdot q_{\mathrm{proxy}}(v)$ and $q_{\mathrm{proxy}}(v)\le 1$ by construction.
\end{proof}

\subsection{Proof of Theorem~\ref{thm:lossless} via Equivalence Classes}
\label{app:global-lossless}

We now turn to the main result. 
For any unpruned tree $\mathcal{U}$ and any positive integer $R$, the pruned tree $T_R=\mathrm{Prune}_R(\mathcal{U})$ is a deterministic function of $\mathcal{U}$. 
We define the equivalence class
\begin{equation*}
    [\mathcal{U}]_R \;=\; \bigl\{\mathcal{U}'\mid \mathrm{Prune}_R(\mathcal{U}')=T_R\bigr\},
\end{equation*}
which collects all unpruned trees that prune to the same $R$-node tree. 
Let $\text{Rheo}([\mathcal{U}]_R)$ denote the random output sequence obtained by running Algorithm~\ref{alg:verification} layer-by-layer on $T_R$. The randomness of $\text{Rheo}([\mathcal{U}]_R)$ is fully from the internal coin flips of Algorithm~\ref{alg:verification}. As in Theorem~\ref{thm:lossless}, once the output leaves $T_R$ (the emitted token is a leaf or not a node of $T_R$), the output is continued autoregressively from $p$, so $\text{Rheo}([\mathcal{U}]_R)$ is defined for every length; throughout, $\mathbf{Seq}=(v_1,\dots,v_L)$ has an arbitrary fixed length $L$ and the event $\{\text{Rheo}=\mathbf{Seq}\}$ refers to the first $L$ output tokens.

Our goal is to prove that, for any sequence $\mathbf{Seq}$ and any $R\ge 1$,
\begin{equation}
    \mathbb{E}_{[\mathcal{U}]_R}\Bigl[\Pr\bigl(\text{Rheo}([\mathcal{U}]_R)=\mathbf{Seq}\bigr)\Bigr]
    \;=\; p(\mathbf{Seq}),
    \label{eq:global-goal}
\end{equation}
where $p(\mathbf{Seq})$ is the probability assigned to $\mathbf{Seq}$ by standard autoregressive sampling from the target model.

\begin{theorem}[Losslessness]
\label{thm:global-lossless}
Equation~\eqref{eq:global-goal} holds for every positive integer $R$ and every sequence $\mathbf{Seq}$.
\end{theorem}

\begin{proof}
We proceed by mathematical induction on $R$.

\paragraph{Base case ($R=1$).} 
The pruned tree $T_1$ contains exactly one node, which must reside at depth $1$ by Lemma~\ref{lem:ancestor}. 
Two sub-cases arise:

\emph{Case $m=0$.} The retained node is the sampled token. 
Its proxy probability is $q(x_1)+\epsilon$, so it is ranked first among the depth-$1$ candidates: no fill token exceeds $q(x_1)$, and any tie is broken in favor of the sampled slot.
Consequently, Algorithm~\ref{alg:verification} performs standard rejection sampling with proposal $\tilde{q}$ and target $p$. The losslessness is trivial.

\emph{Case $m>0$.} The retained node is the top-$1$ deterministic candidate. 
Algorithm~\ref{alg:verification} sets $s=-1$ and $r\leftarrow p$, so the output is sampled directly from $p$.

In both cases, the expectation $\mathbb{E}_{[\mathcal{U}]_1}\left[\Pr(\text{Rheo}([\mathcal{U}]_1)=\mathbf{Seq})\right] = p(\mathbf{Seq})$.

\paragraph{Inductive hypothesis.} 
Assume that for $R = k$ and all sequences $\mathbf{Seq}$,
\begin{equation}
    \mathbb{E}_{[\mathcal{U}]_R}\Bigl[\Pr\bigl(\text{Rheo}([\mathcal{U}]_R)=\mathbf{Seq}\bigr)\Bigr]
    \;=\; p(\mathbf{Seq}).
    \label{eq:induction-hypothesis}
\end{equation}

\paragraph{Inductive step ($R=k+1$).} 
Fix an arbitrary equivalence class $[\mathcal{U}]_k$ and consider all unpruned trees $\mathcal{U}'\in[\mathcal{U}]_k$. 
For each such $\mathcal{U}'$, let $u'$ be the unique node in $\mathrm{Prune}_{k+1}(\mathcal{U}')\setminus\mathrm{Prune}_k(\mathcal{U}')$ (the newly admitted $(k+1)$-th node). 
The following structural fact is essential and its proof is deferred to Appendix~\ref{app:slot-proof}.

\begin{lemma}[Slot determinism]
\label{lem:slot}
For any equivalence class $[\mathcal{U}]_k$, if the node $u'=\mathrm{Prune}_{k+1}(\mathcal{U}')\setminus\mathrm{Prune}_k(\mathcal{U}')$ exists for some $\mathcal{U}'\in[\mathcal{U}]_k$, then
\begin{enumerate}
    \item \textbf{(Topological invariance)} The slot of $u'$ are identical across all $\mathcal{U}'\in[\mathcal{U}]_k$, which means that $\mathrm{Prune}_{k+1}(\mathcal{U}')$ have the same topological structure for all $\mathcal{U}'\in[\mathcal{U}]_k$.
    \item \textbf{(Deterministic consistency)} If the slot type of $u'$ is deterministic (top-$m$ slot or fill slot), then $\mathrm{Prune}_{k+1}(\mathcal{U}')$ is identical for every $\mathcal{U}'\in[\mathcal{U}]_k$, which means $[\mathcal{U}]_{k+1}\subset [\mathcal{U}]_k$.
    \item \textbf{(Stochastic fidelity)} If the slot type of $u'$ is sampled, then $[\mathcal{U}]_k = \bigsqcup_{v}\, [\mathcal{U}]_k^{(v)}$ are partitioned into disjoint subclasses according to the realized token $v$ of $u'$, and the proportion of each subclass equals $\tilde{q}(v)$.
\end{enumerate}
\end{lemma}

By Lemma~\ref{lem:ancestor} and the tie-breaking rule, $u'$ is a leaf in $T_{k+1}=\mathrm{Prune}_{k+1}(\mathcal{U}')$, so it affects only the single layer $\ell$ where it resides. 
Let $T_k=\mathrm{Prune}_k(\mathcal{U}')$ (which is constant for all $\mathcal{U}'\in[\mathcal{U}]_k$ by definition) and let $C$ be the candidate set of $T_k$ at layer $\ell$.

We now split the analysis according to the slot type of $u'$.

\emph{Case 1: deterministic slot (top-$m$ or fill).} 
By Lemma~\ref{lem:slot}(2), the pruned tree $T_{k+1}$ is identical for every $\mathcal{U}'\in[\mathcal{U}]_k$. 
During RheoSampling, the verification on $T_k$ and $T_{k+1}$ proceeds identically until layer $\ell$. 
Algorithm~\ref{alg:verification} processes $u'$ as an additional point mass in the sequential-rejection pool. Whether $u'$ is accepted directly or bypassed, the total probability mass allocated to each token remains governed by the target distribution $r$ (or $p$). The presence of $u'$ alters only the acceptance length, not the marginal output distribution at layer $\ell$. Moreover, since $u'$ is a leaf of $T_{k+1}$, by the same continuation argument as in Case 2 below, the conditional law of the tokens after layer $\ell$ given the emitted token is identical for $T_k$ and $T_{k+1}$.
Hence, for every $\mathcal{U}'\in[\mathcal{U}]_k$,
\begin{equation}
    \Pr\bigl(\text{Rheo}([\mathcal{U}']_{k+1})=\mathbf{Seq}\bigr)
    \;=\;
    \Pr\bigl(\text{Rheo}([\mathcal{U}]_k)=\mathbf{Seq}\bigr).
    \label{eq:case1-equation}
\end{equation}

\emph{Case 2: sampled slot.} 
By Lemma~\ref{lem:slot}(3), the equivalence class $[\mathcal{U}]_k$ is partitioned into subclasses $[\mathcal{U}]_k^{(v)}$ indexed by the realized token $v$ of $u'$, with $\Pr(\mathcal{U}'\in[\mathcal{U}]_k^{(v)})=\tilde{q}(v)$. 
For each $v$, the pruned tree is $T_{k+1}^{(v)}=T_k\cup\{u'_v\}$, where $u'_v$ denotes the sampled node filled with token $v$. 
At layer $\ell$, $T_k$ has $s=-1$ (the sampled slot was previously empty), while $T_{k+1}^{(v)}$ has $s\neq -1$ pointing to $u'_v$ with true sampling probability $\tilde{q}(v)$.

Algorithm~\ref{alg:verification} first tests $u'_v$ with acceptance probability $\min(1,p(v)/\tilde{q}(v))$. 
If accepted, the output at layer $\ell$ is $v$; if rejected, the residual distribution
\begin{equation*}
    r \;=\; \mathrm{norm}\bigl(\max(0,\,p-\tilde{q})\bigr)
\end{equation*}
is verified over the deterministic candidate set $C$. 
Let $h_{\mathrm{res}}(C,v_\ell)$ denote the probability that Algorithm~\ref{alg:verification}, operating on the deterministic set $C$ under target distribution $r$, outputs token $v_\ell$ (the token appearing in $\mathbf{Seq}$ at layer $\ell$). 
A standard sequential-rejection argument shows $h_{\mathrm{res}}(C,v_\ell)=r(v_\ell)$.

Averaging over the realized token $v\sim\tilde{q}$, the probability of emitting $v_\ell$ at layer $\ell$ is
\begin{align}
    &\sum_{v}\tilde{q}(v)\Bigl[\mathbb{I}(v=v_\ell)\min\!\Big(1,\frac{p(v)}{\tilde{q}(v)}\Big)
    \;+\;
    \Bigl(1-\min\Big(1,\frac{p(v)}{\tilde{q}(v)}\Big)\Bigr)\,h_{\mathrm{res}}(C,v_\ell)\Bigr] \notag\\
    &=\; \min\bigl(p(v_\ell),\tilde{q}(v_\ell)\bigr)
    \;+\;
    h_{\mathrm{res}}(C,v_\ell)\sum_{v}\max\bigl(0,\tilde{q}(v)-p(v)\bigr) \notag\\
    &=\; \min\bigl(p(v_\ell),\tilde{q}(v_\ell)\bigr)
    \;+\;
    \frac{\max\bigl(0,p(v_\ell)-\tilde{q}(v_\ell)\bigr)}{\sum_t\max(0,p(t)-\tilde{q}(t))}
    \cdot
    \sum_{v}\max\bigl(0,p(v)-\tilde{q}(v)\bigr) \notag\\
    &=\; \min\bigl(p(v_\ell),\tilde{q}(v_\ell)\bigr)
    \;+\;
    \max\bigl(0,p(v_\ell)-\tilde{q}(v_\ell)\bigr)
    \;=\; p(v_\ell).
    \label{eq:sample-expectation}
\end{align}
Meanwhile, $T_k$ at layer $\ell$ (with $s=-1$ and $r\leftarrow p$) also emits $v_\ell$ with probability $p(v_\ell)$.

It remains to check that the continuation after layer $\ell$ is also unchanged. The two trees $T_k$ and $T_{k+1}^{(v)}$ coincide everywhere except for the leaf $u'_v$, so the verification before layer $\ell$ is identical. Conditional on emitting a token $w$ at layer $\ell$: if $w\in C$, both trees descend into the same subtree of $w$ (which does not contain $u'_v$); otherwise $w$ is either the leaf $u'_v$ of $T_{k+1}^{(v)}$ or not a node of either tree (note that $v\notin C$, since the sampled token is distinct from its deterministic siblings), and in both trees the remaining tokens are drawn autoregressively from $p(\cdot\mid v_{1:\ell-1},w)$. Hence the conditional law of the tokens after layer $\ell$ given $(v_{1:\ell-1},w)$ is identical for $T_k$ and $T_{k+1}^{(v)}$, and combining it with the equality of the layer-$\ell$ marginals above, we obtain
\begin{equation}
    \mathbb{E}_{\mathcal{U}'\in[\mathcal{U}]_k}\Bigl[\Pr\bigl(\text{Rheo}([\mathcal{U}']_{k+1})=\mathbf{Seq}\bigr)\Bigr]
    \;=\;
    \Pr\bigl(\text{Rheo}([\mathcal{U}]_k)=\mathbf{Seq}\bigr).
    \label{eq:case2-expectation}
\end{equation}

\paragraph{Completing the induction.} 
Combining Eq.~\eqref{eq:case1-equation} of Case 1 and Eq.~\eqref{eq:case2-expectation} of Case 2, for every equivalence class $[\mathcal{U}]_k$ we have
\begin{equation}
    \mathbb{E}_{\mathcal{U}'\in[\mathcal{U}]_k}\Bigl[\Pr\bigl(\text{Rheo}([\mathcal{U}']_{k+1})=\mathbf{Seq}\bigr)\Bigr]
    \;=\;
    \Pr\bigl(\text{Rheo}([\mathcal{U}]_k)=\mathbf{Seq}\bigr).
\end{equation}
Taking expectation over all $[\mathcal{U}]_k$,
\begin{align}
    \mathbb{E}_{[\mathcal{U}]_{k+1}}\Bigl[\Pr\bigl(\text{Rheo}([\mathcal{U}]_{k+1})=\mathbf{Seq}\bigr)\Bigr]
    &=\;
    \mathbb{E}_{[\mathcal{U}]_k}\Bigl[\mathbb{E}_{\mathcal{U}'\in[\mathcal{U}]_k}\bigl[\Pr(\text{Rheo}([\mathcal{U}']_{k+1})=\mathbf{Seq})\bigr]\Bigr] \notag\\
    &=\;
    \mathbb{E}_{[\mathcal{U}]_k}\Bigl[\Pr\bigl(\text{Rheo}([\mathcal{U}]_k)=\mathbf{Seq}\bigr)\Bigr] \notag\\
    &=\; p(\mathbf{Seq}),
\end{align}
where the last equality is the induction hypothesis~\eqref{eq:induction-hypothesis}. 
Thus Eq.~\eqref{eq:global-goal} holds for $R=k+1$.

By mathematical induction, Theorem~\ref{thm:global-lossless} holds for every positive integer $R$. Further, Theorem~\ref{thm:lossless} holds and RheoSampling is lossless.
\end{proof}

\subsection{Deferred Proof of Lemma~\ref{lem:slot} (Slot Determinism)}
\label{app:slot-proof}

\begin{proof}[Proof of Lemma~\ref{lem:slot}]

We prove the three statements in order.

\paragraph{Step 1: Existence of the $(k+1)$-th slot in every $\mathcal{U}\in[\mathcal{U}]_k$.}
Fix $\mathcal{U}'\in[\mathcal{U}]_k$ and let $u'=\mathrm{Prune}_{k+1}(\mathcal{U}')\setminus\mathrm{Prune}_k(\mathcal{U}')$. 
Let $P'$ be the parent of $u'$. If $P'$ is the root, it is expanded in every tree and the slot of $u'$ trivially exists; otherwise $P'\in\mathrm{Prune}_k(\mathcal{U}')=T_k$ by Lemma~\ref{lem:ancestor}.
Because every $\mathcal{U}\in[\mathcal{U}]_k$ shares the same $T_k$, it follows that $P'\in\mathrm{Prune}_k(\mathcal{U})$ for all $\mathcal{U}\in[\mathcal{U}]_k$.

We claim that $P'$ is expanded in every $\mathcal{U}\in[\mathcal{U}]_k$. 
Suppose not: there exists some $\mathcal{U}\in[\mathcal{U}]_k$ in which $P'$ is \emph{not expanded}. 
At depth $\ell=\mathrm{depth}(P')$, the tree-construction policy expands exactly the $K$ nodes with highest path scores. 
Since $P'$ is expanded in $\mathcal{U}'$ (because its child $u'$ exists), it belongs to the top-$K$ at depth $\ell$ in $\mathcal{U}'$. 
Since $P'$ is \emph{not expanded} in $\mathcal{U}$, there must be $K$ other nodes at depth $\ell$ in $\mathcal{U}$ which are higher priority than $P'$. 
Because $P'\in T_k$ and $T_k$ collects the globally highest $k$ nodes in $\mathcal{U}$, any node with higher priority than $P'$ must also belong to $T_k$. 
Thus these $K$ nodes all lie in $T_k=\mathrm{Prune}_k(\mathcal{U}')$, and therefore appear at depth $\ell$ in $\mathcal{U}'$ as well, with the same higher priority. 
Consequently, in $\mathcal{U}'$ the node $P'$ is ranked below at least $K$ nodes at depth $\ell$ and cannot be expanded. This creates a contradiction since $P'$ has a child $u'$ in $\mathcal{U}'$. 
Hence $P'$ is expanded in every $\mathcal{U}\in[\mathcal{U}]_k$, and the child slot occupied by $u'$ exists in all of them.

\paragraph{Step 2: Score invariance and uniqueness.}
Fix any $\mathcal{U}\in[\mathcal{U}]_k$, let $u$ be its $(k+1)$-th node, with parent $P$, which is either the root or a node of $T_k$.
In both cases, its path score $S_{T_k}(P)$ is frozen by $T_k$ (the root has score $1$) and is therefore identical across $[\mathcal{U}]_k$. Let ${\rm slot}(u)$ be the corresponding slot of $u$ in $\mathcal{U}$, then we inspect the three possible slot types of ${\rm slot}(u)$ to prove its path score $S_{\mathcal{U}}({\rm slot}(u))$ is also frozen by $T_k$:

\emph{Top-$m$ slot.} 
The token $x_i\rightarrow {\rm slot}(u)$ and its draft probability $q(x_i)$ are deterministic functions of the parent. 
Hence the slot score $S_{\mathcal{U}}({\rm slot}(u))=S_{T_k}(P)\cdot q(x_i)$ is a constant over $[\mathcal{U}]_k$.

\emph{Sampled slot.} 
The proxy probability $q_{\mathrm{proxy}}(Y)=\min\{q(x_m),z\}$ depends only on the top-$m$ probabilities and the residual mass, all of which are fixed before sampling. 
Thus the slot score $S_{\mathcal{U}}({\rm slot}(u))=S_{T_k}(P)\cdot q_{\mathrm{proxy}}(Y)$ is likewise constant over all $\mathcal{U}\in[\mathcal{U}]_k$, independent of the realized token $Y$.

\emph{Fill slot.} 
By Lemma~\ref{lem:local-rank}, the sampled sibling always has proxy score at least $q(x_{m+1})$, dominating every fill token including $u$. 
Therefore the sampled sibling of $u$ must already belong to $T_k$, which freezes its realized token and hence the remaining vocabulary is deterministic.
Consequently the fill token and its slot score $S_{\mathcal{U}}({\rm slot}(u))$ are identical across the whole equivalence class $\mathcal{U}\in[\mathcal{U}]_k$.

Now fix $\mathcal{U},\mathcal{U}'\in[\mathcal{U}]_k$ and let $u,u'$ be their respective $(k+1)$-th nodes. 
By Step~1, the slot of $u$ exists in $\mathcal{U}'$ and the slot of $u'$ exists in $\mathcal{U}$. 
Let $\sigma$ and $\sigma'$ denote these two slots, respectively. Then we create a particular tree $\mathcal{T}=T_k\cup\{\sigma,\sigma'\}$ with tokens $u\rightarrow\sigma$ and $u'\rightarrow\sigma'$. 
In $\mathcal{U}$, since $u$ is the $(k+1)$-th node, slot $\sigma'$ cannot outrank $\sigma$; hence path score $S_{\mathcal{U}}(\sigma')\le S_{\mathcal{U}}(\sigma)$ and then $S_{\mathcal{T}}(\sigma')\le S_{\mathcal{T}}(\sigma)$. 
Symmetrically, in $\mathcal{U}'$, we have $S_{\mathcal{U}'}(\sigma)\le S_{\mathcal{U}'}(\sigma')$ and then $S_{\mathcal{T}}(\sigma)\le S_{\mathcal{T}}(\sigma')$. 
Therefore $S_{\mathcal{T}}(\sigma') = S_{\mathcal{T}}(\sigma)$ in tree $\mathcal{T}$. 
Because the global tie-breaking rule (depth first, then left-to-right) is deterministic and fixed, the same slot wins the $(k+1)$-th rank in these three trees $\mathcal{U},\mathcal{U}'$ and $\mathcal{T}$. 
Hence $\sigma=\sigma'$, establishing topological invariance~(1).

\paragraph{Step 3: Deterministic consistency (2).}
If the slot type of $u'$ is top-$m$ or fill, Step~2 shows that its token and score are both frozen by $T_k$. 
Hence $\mathrm{Prune}_{k+1}(\mathcal{U}')$ is identical for every $\mathcal{U}'\in[\mathcal{U}]_k$.

\paragraph{Step 4: Stochastic fidelity (3).}
If the slot of $u'$ is sampled, its proxy score is constant (Step~2), while the realized token $Y$ is drawn from $\tilde{q}$ during tree construction.
Moreover, $T_k$ contains no node whose token or score depends on $Y$: such nodes are the fill siblings of this slot and the descendants of this slot, all of which rank below the slot itself (Lemma~\ref{lem:local-rank}, Lemma~\ref{lem:ancestor}, and the tie-breaking rule) and hence lie outside $T_k$; and, as in Step~1, whether a node of $T_k$ is expanded is determined only by higher-priority nodes, which also lie in $T_k$.
Thus membership in $[\mathcal{U}]_k$ does not depend on $Y$, and since $Y$ is sampled independently of the remaining randomness that determines $T_k$, conditioning on $[\mathcal{U}]_k$ leaves $Y\sim\tilde{q}$.
Because the survival of this slot in the global top-$(k+1)$ depends only on the constant proxy score and not on realized sample token $u'$, the equivalence class $[\mathcal{U}]_k$ is partitioned into subclasses $[\mathcal{U}]_k^{(v)}=\{\mathcal{U}'\in[\mathcal{U}]_k\mid u'=v\}$ whose proportions are exactly $\tilde{q}(v)$.
\end{proof}

\subsection{Numerical Verification on an Enumerable Tree}
\label{app:numerical}

Downstream accuracy parity is not a direct test of distributional fidelity, and sequence-level divergences cannot be estimated reliably on real LLM outputs. We therefore complement the proof with a controlled finite-state experiment in which both the exact target sequence distribution and every stochastic tree realization can be enumerated.

\paragraph{Setup.}
We use a vocabulary $\{A,B,C,D,E\}$, $K=3$, $m=1$, draft depth $2$, and a global pruning budget of $R=6$ out of the $3+3\times3=12$ unpruned nodes. The draft and target distributions at the root and after each first token are listed in Table~\ref{tab:toy-dist}. At every parent, the candidate pool consists of the top-$1$ token, one token sampled from $\tilde{q}$, and the highest-probability remaining fill token; the sampled slot receives $q_{\mathrm{proxy}}=\min\{q(x_1),z\}$. All three first-layer nodes are expanded by the same rule, with independent tail samples, after which global top-$6$ pruning is applied with the tie-breaking rule of Appendix~\ref{app:setup}. Algorithm~\ref{alg:verification} is then executed layer by layer. If the first-layer output is not a node of the pruned tree, or the accepted node has no surviving child, the second token is drawn from $p(\cdot\mid\text{first token})$, which is equivalent to starting a fresh lossless round at that prefix.

\begin{table}[h]
\centering
\caption{Draft ($q$) and target ($p$) distributions of the toy experiment. The first row is the root distribution; each subsequent row is conditioned on the first token.}
\label{tab:toy-dist}
\small
\begin{tabular}{@{}l|ccccc|ccccc@{}}
\toprule
 & \multicolumn{5}{c|}{$q$} & \multicolumn{5}{c}{$p$} \\
Prefix & $A$ & $B$ & $C$ & $D$ & $E$ & $A$ & $B$ & $C$ & $D$ & $E$ \\
\midrule
$\varnothing$ & 0.40 & 0.25 & 0.15 & 0.12 & 0.08 & 0.20 & 0.10 & 0.30 & 0.25 & 0.15 \\
$A$ & 0.42 & 0.24 & 0.16 & 0.11 & 0.07 & 0.18 & 0.12 & 0.30 & 0.25 & 0.15 \\
$B$ & 0.38 & 0.27 & 0.14 & 0.13 & 0.08 & 0.25 & 0.10 & 0.20 & 0.30 & 0.15 \\
$C$ & 0.34 & 0.29 & 0.18 & 0.11 & 0.08 & 0.12 & 0.28 & 0.18 & 0.22 & 0.20 \\
$D$ & 0.46 & 0.20 & 0.15 & 0.12 & 0.07 & 0.30 & 0.08 & 0.24 & 0.18 & 0.20 \\
$E$ & 0.36 & 0.23 & 0.19 & 0.13 & 0.09 & 0.16 & 0.20 & 0.14 & 0.28 & 0.22 \\
\bottomrule
\end{tabular}
\end{table}

\paragraph{Procedure.}
Each pool admits $4$ possible tail samples, so the four pools (root and three first-layer nodes) yield $4^4=256$ tree realizations, each weighted by the product of its $\tilde{q}$ probabilities. For every realization we integrate the outcomes of Algorithm~\ref{alg:verification} analytically in exact rational arithmetic, which gives the exact induced distribution $\hat{p}$ over all $25$ length-two sequences; we compare it with the autoregressive target $p(a)\,p(b\mid a)$. As a diagnostic control, we repeat the computation with the invalid identity-dependent proxy $q_{\mathrm{proxy}}=q(Y)$. For reference, we also run $10^5$ Monte Carlo trials of both the full RheoSampling pipeline and vanilla target sampling, with fixed seeds $20261024$ (vanilla) and $20261025$ (RheoSampling); these Monte Carlo rows vary at the $10^{-3}$ level across seeds. A reference implementation is provided in the supplementary material.

\begin{table}[h]
\centering
\caption{Sequence-level fidelity on the enumerable toy tree ($R=6$).}
\label{tab:toy-fidelity}
\small
\begin{tabular}{@{}lccc@{}}
\toprule
Method & TV & $\mathrm{KL}(\hat{p}\,\|\,p)$ & Max error \\
\midrule
Exact RheoSampling & $0$ & $0$ & $0$ \\
Monte Carlo, vanilla target sampling & $8.09\times10^{-3}$ & $2.08\times10^{-4}$ & $2.06\times10^{-3}$ \\
Monte Carlo, RheoSampling & $7.82\times10^{-3}$ & $2.05\times10^{-4}$ & $1.76\times10^{-3}$ \\
Exact, invalid proxy $q(Y)$ & $3.52\times10^{-2}$ & $1.41\times10^{-2}$ & $2.11\times10^{-2}$ \\
\bottomrule
\end{tabular}
\end{table}

\paragraph{Results.}
As shown in Table~\ref{tab:toy-fidelity}, the exact distribution produced by RheoSampling coincides with the target as rational numbers on all $25$ sequences, and its Monte Carlo deviation is comparable to that of vanilla target sampling under the same budget. The same exact equality holds for every budget $R\in\{1,\dots,12\}$. In contrast, the invalid proxy $q(Y)$ incurs a clear systematic bias for every $2\le R\le 11$; the bias vanishes only in the degenerate cases $R=1$ (only the top-$1$ node survives) and $R=12$ (no pruning), where survival is trivially independent of $Y$. This confirms that the experiment is sensitive to violations of the losslessness condition, and that the identity-independent proxy of Lemma~\ref{lem:local-rank} is what preserves the target distribution under global pruning.

\section{The Superiority of RheoVerification}

\subsection{Single-layer Rheo-acceptance Rate (Proof of Theorem~\ref{thm:rheo-acc})}
\label{app:acceptance}

\textbf{Setup.} Let $x_1,x_2,\dots$ be the vocabulary sorted by draft probability $q$ in descending order.
A candidate pool of size $n$ consists of:
\begin{itemize}
    \item top-$m$ deterministic tokens $\{x_1,\dots,x_m\}$;
    \item one sampled token $Y\sim\tilde{q}$, where $\tilde{q}(v)=q(v)/z$ for $v\notin\{x_1,\dots,x_m\}$ and $z=1-\sum_{i=1}^m q(x_i)$;
    \item $n-m-1$ fill tokens, i.e. the highest-$q$ tokens from $\mathcal{V}\setminus\{x_1,\dots,x_m,Y\}$.
\end{itemize}
Throughout this section (including Appendix~\ref{app:rrsw}), $\Pr[\text{output}=v]$ denotes the probability that $v$ is output as an \emph{accepted candidate}, excluding the final fallback sample of Algorithm~\ref{alg:verification}; the acceptance rate is the sum of these probabilities over $v$.

\textbf{Always-present tokens.}
The first $n-1$ tokens $\mathrm{Top}_{n-1}=\{x_1,\dots,x_{n-1}\}$ appear in \emph{every} realization of the pool:
\begin{itemize}
    \item $x_1,\dots,x_m$ are deterministic;
    \item each $x_j$ ($m+1\le j\le n-1$) is either the sampled token (if $Y=x_j$) or a fill token (if $Y\neq x_j$).
\end{itemize}

\textbf{Output probability of $v\in\mathrm{Top}_{n-1}$.}
For any $v\in\mathrm{Top}_{n-1}$, Algorithm~\ref{alg:verification} outputs $v$ via two paths:
\begin{align}
    \Pr[\text{output}=v]
    &=\underbrace{\tilde{q}(v)\min\Bigl(1,\frac{p(v)}{\tilde{q}(v)}\Bigr)}_{Y=v\text{, direct accept}}
    \;+\;
    \underbrace{r(v)\sum_{t\neq v}\tilde{q}(t)\max\Bigl(0,1-\frac{p(t)}{\tilde{q}(t)}\Bigr)}_{Y\neq v\text{, reject }Y\text{ then hit }v\text{ under }r} \notag\\
    &=\min\bigl(p(v),\tilde{q}(v)\bigr)
    \;+\;
    \frac{\max\bigl(0,p(v)-\tilde{q}(v)\bigr)}{Z}\cdot\bigl(Z-\max(0,\tilde{q}(v)-p(v))\bigr) \notag\\
    &=p(v),
    \label{eq:top-n-1-exact}
\end{align}
where $Z=\sum_w\max(0,\tilde{q}(w)-p(w))=\sum_w\max(0,p(w)-\tilde{q}(w))$ and $r=\mathrm{norm}(\max(0,p-\tilde{q}))$.
The last equality follows from the identity $\min(a,b)+\max(0,a-b)=a$.

Summing Eq.~\eqref{eq:top-n-1-exact} over $v\in\mathrm{Top}_{n-1}$ yields the first term of Theorem~\ref{thm:rheo-acc}.

\textbf{Output probability of $v\notin\mathrm{Top}_{n-1}$.}
For such $v$, two cases arise according to whether $v=x_n$.

\emph{Case 3a: $v=x_n$.}
This token enters the pool as a fill token \emph{iff} $Y\in\{x_{m+1},\dots,x_{n-1}\}$ (as the last fill token), and as the sampled token iff $Y=x_n$; otherwise it is absent.
Its output probability therefore decomposes as
\begin{align}
    \Pr[\text{output}=x_n]
    &=\underbrace{\tilde{q}(x_n)\min\Bigl(1,\frac{p(x_n)}{\tilde{q}(x_n)}\Bigr)}_{\text{sampled path}}
    \;+\;
    \underbrace{r(x_n)\sum_{j=m+1}^{n-1}\tilde{q}(x_j)\max\Bigl(0,1-\frac{p(x_j)}{\tilde{q}(x_j)}\Bigr)}_{\text{fill path}} \notag\\
    &=\min\bigl(p(x_n),\tilde{q}(x_n)\bigr)
    \;+\;
    r(x_n)\sum_{v\in\mathrm{Top}_{n-1}}\max\bigl(0,\tilde{q}(v)-p(v)\bigr).
    \label{eq:xn-contrib}
\end{align}

\emph{Case 3b: $v\notin\mathrm{Top}_n$.}
Such a token can only appear as the sampled token $Y=v$, and it never survives into the fill set because the fill slots are exhausted by $\mathrm{Top}_{n-1}\cup\{x_n\}\setminus\{Y\}$.
Hence
\begin{equation}
    \Pr[\text{output}=v]=\tilde{q}(v)\min \Bigl(1,\frac{p(v)}{\tilde{q}(v)}\Bigr)=\min \bigl(p(v),\tilde{q}(v)\bigr).
    \label{eq:tail-contrib}
\end{equation}

\textbf{Summation.}
Adding Eqs.~\eqref{eq:top-n-1-exact}, \eqref{eq:xn-contrib}, and \eqref{eq:tail-contrib} over their respective token sets gives
\begin{equation}
    \mathcal{A}_{\text{Rheo}}
    \;=\;
    \sum_{v\in\mathrm{Top}_{n-1}}p(v)
    \;+\;
    \sum_{v\notin\mathrm{Top}_{n-1}}\min\bigl(p(v),\tilde{q}(v)\bigr)
    \;+\;
    r(x_n)\sum_{v\in\mathrm{Top}_{n-1}}\max\bigl(0,\tilde{q}(v)-p(v)\bigr),
\end{equation}
which establishes the equality of \eqref{eq:rheo-acc}. The inequality of \eqref{eq:rheo-acc} is trivial since
\[\sum_{v\in\mathrm{Top}_{n-1}}\max\bigl(0,\tilde{q}(v)-p(v)\bigr) \le Z \;\;\;\;\text{and}\;\;\;\;p(x_n)=\min\bigl(p(x_n),\tilde{q}(x_n)\bigr)+
r(x_n)\cdot Z.\]
\hfill$\square$


\subsection{Single-layer RRSw-acceptance Rate}
\label{app:rrsw}

\textbf{Setup (RRSw baseline).}
The candidate pool is identical to that in Theorem~\ref{thm:rheo-acc}, but the verification order is changed. The sampled token $Y\sim\tilde{q}$ is evaluated \emph{after} the deterministic top-$m$ tokens and \emph{before} fill tokens.

Let $\tilde{p}$ be the renormalized restriction of $p$ to the complement 
of $\mathrm{Top}_m=\{x_1,\dots,x_m\}$. 
The verifier first performs sequential rejection over $\mathrm{Top}_m$ 
under $p$; if all are rejected, it tests $Y$ under $\tilde{p}$; 
if $Y$ is also rejected, the remaining fill token $x_n$ is tested 
under $r_{\rm RRSw} = \operatorname{norm}(\max(0,\tilde{p} - \tilde{q}))$.
Trivially, sequential rejection over $\{x_1,\dots,x_m\}$ under the original target $p$ recovers
\begin{equation}\label{eq:topm-acc}
    A:=\sum_{i=1}^m p(x_i) = \sum\limits_{v\in{\rm Top}_{m}} p(v).
\end{equation}

\textbf{Always-present tail tokens.}
For each $v\in\{x_{m+1},\dots,x_{n-1}\}$, two paths lead to acceptance:
\begin{align}
    \Pr[\text{output}=v]
    &=(1-A)\cdot\left[\underbrace{\tilde{q}(v)\min\Bigl(1,\frac{\tilde p(v)}{\tilde{q}(v)}\Bigr)}_{Y=v\text{, direct accept}}
    \;+\;
    \underbrace{r_{\rm RRSw}(v)\sum_{t\neq v}\tilde{q}(t)\max\Bigl(0,1-\frac{\tilde p(t)}{\tilde{q}(t)}\Bigr)}_{Y\neq v\text{, reject }Y\text{ then hit }v\text{ under }r_{\rm RRSw}} \right] \notag\\
    &= (1-A) \left[\min\!\bigl(\tilde p(v),\tilde{q}(v)\bigr)
    +
    \frac{\max\!\bigl(0,\tilde p(v)-\tilde{q}(v)\bigr)}{\tilde Z} \bigl(\tilde Z-\max\!(0,\tilde{q}(v)-\tilde p(v))\bigr)\right] \notag\\
    &=(1-A)\cdot\tilde p(v) = p(v),
    \label{eq:rrsw-top-n-1}
\end{align}
where $\tilde Z=\sum_w\max(0,\tilde{q}(w)-\tilde p(w))=\sum_w\max(0,\tilde p(w)-\tilde{q}(w))$.

Together with Eq.~\eqref{eq:topm-acc}, the first $n-1$ tokens contribute $\sum_{v\in\mathrm{Top}_{n-1}}p(v)$.

\textbf{Output probability of $v\notin\mathrm{Top}_{n-1}$.} For such $v$, two cases arise according to whether $v=x_n$.

\emph{Case 3a: $v=x_n$.}
This token enters the pool as a fill token iff $Y\in\{x_{m+1},\dots,x_{n-1}\}$ (as the last fill token), and as the sampled token iff $Y=x_n$. 
Conditioned on the prefix $\mathrm{Top}_m$ being rejected (probability $1-A$), its output probability therefore decomposes as
\begin{align}
    \Pr[\text{output}=x_n]
    &=(1-A)\left[\underbrace{\min\!\Bigl(\tilde{q}(x_n),\tilde p(x_n)\Bigr)}_{\text{sampled path}}
    +
    \underbrace{r_{\rm RRSw}(x_n)\!\sum_{j=m+1}^{n-1}\tilde{q}(x_j)\max\!\Bigl(0,1-\frac{\tilde p(x_j)}{\tilde{q}(x_j)}\Bigr)}_{\text{fill path}}\right] \notag\\
    &= (1-A) \left[\min\!\bigl(\tilde p(x_n), \tilde{q}(x_n)\bigr)
    \;+\;
    r_{\rm RRSw}(x_n)\!\sum_{v\in\mathrm{Top}_{n-1}}\max\!\bigl(0,\tilde{q}(v)-\tilde p(v)\bigr)\right].
    \label{eq:rrsw-xn-case3a}
\end{align}

\emph{Case 3b: $v\notin\mathrm{Top}_n$.}
Such a token can only appear as the sampled token $Y=v$. 
Therefore
\begin{equation}
    \Pr[\text{output}=v]=(1-A)\tilde{q}(v)\min\Bigl(1,\frac{\tilde{p}(v)}{\tilde{q}(v)}\Bigr) = (1-A)\min\bigl(\tilde p(v),\tilde{q}(v)\bigr).
    \label{eq:rrsw-xn-case3b}
\end{equation}

\textbf{Summation.}
Adding the contributions from Eqs.~\eqref{eq:topm-acc}, \eqref{eq:rrsw-top-n-1}, \eqref{eq:rrsw-xn-case3a} and \eqref{eq:rrsw-xn-case3b} yields
\begin{align}
    \mathcal{A}_{\rm RRSw}
    = &\sum_{v\in\mathrm{Top}_{n-1}}p(v)
    + \sum_{v\notin\mathrm{Top}_{n-1}}\min\bigl(p(v),(1-A)\tilde{q}(v)\bigr) \notag \\
    &~+ (1-A)\cdot r_{\rm RRSw}(x_n)\sum_{v\in\mathrm{Top}_{n-1}}\max\bigl(0,\tilde{q}(v)-\tilde p(v)\bigr),
    \label{eq:rrsw-acc}
\end{align}
where $r_{\rm RRSw}=\operatorname{norm}(\max(0,\tilde{p}-\tilde{q}))$ and $A=\sum\limits_{v\in{\rm Top}_{m}} p(v)$.

\subsection{Approximate Comparison: Rheo vs. RRSw}
\label{app:comparison}

\textbf{Approximation assumption.}
In both acceptance-rate formulas, the margin-token terms involve a finite sum over $\mathrm{Top}_{n-1}$:
\begin{equation*}
    S_{\rm Rheo}=\sum_{v\in\mathrm{Top}_{n-1}}\max\!\bigl(0,\tilde{q}(v)-p(v)\bigr),
    \qquad
    S_{\rm RRSw}=\sum_{v\in\mathrm{Top}_{n-1}}\max\!\bigl(0,\tilde{q}(v)-\tilde{p}(v)\bigr).
\end{equation*}
Because $\tilde{q}$ is dominated by its top-ranked entries (it is a renormalized tail of the draft distribution), the omitted tail contribution is negligible. We therefore approximate
\begin{equation}
    S_{\rm Rheo}\approx Z:=\sum_{v}\max\!\bigl(0,\tilde{q}(v)-p(v)\bigr),
    \qquad
    S_{\rm RRSw}\approx \tilde{Z}:=\sum_{v}\max\!\bigl(0,\tilde{q}(v)-\tilde{p}(v)\bigr).
    \label{eq:approx}
\end{equation}

\textbf{Case 3a under approximation.}
Plugging Eq.~\eqref{eq:approx} into the margin-token terms:
\begin{align*}
    \text{Rheo:}&\quad \min\!\bigl(p(x_n),\tilde{q}(x_n)\bigr)+r(x_n)S_{\rm Rheo}\\
    &\;\approx\; \min\!\bigl(p(x_n),\tilde{q}(x_n)\bigr)+\max\!\bigl(0,p(x_n)-\tilde{q}(x_n)\bigr)
    \;=\;p(x_n),\\[4pt]
    \text{RRSw:}&\quad
    (1-A) \Bigl[\min\!\bigl(\tilde{p}(x_n),\tilde{q}(x_n)\bigr)+r_{\rm RRSw}(x_n)S_{\rm RRSw}\Bigr]\\
    &\;\approx\;
    (1-A) \Bigl[\min\!\bigl(\tilde{p}(x_n),\tilde{q}(x_n)\bigr)+\max\!\bigl(0,\tilde{p}(x_n)-\tilde{q}(x_n)\bigr)\Bigr]
    \;=\;p(x_n).
\end{align*}
Thus, under the approximation, the margin-token contributions are equal.

\textbf{Case 3b (strict inequality).}
For every $v\notin\mathrm{Top}_n$, the tail-overlap terms satisfy
\begin{equation*}
    \min\!\bigl(p(v),\tilde{q}(v)\bigr)
    \;\ge\;
    \min\!\bigl(p(v),(1-A)\tilde{q}(v)\bigr),
\end{equation*}
with strict inequality whenever $(1-A)\tilde{q}(v)<p(v)<\tilde{q}(v)$.

\textbf{Net difference.}
Collecting the two cases, the approximate acceptance rates become
\begin{align*}
    \mathcal{A}_{\text{Rheo}} &\approx \sum_{v\in\mathrm{Top}_{n-1}}p(v)+p(x_n)+\sum_{v\notin\mathrm{Top}_n}\min\!\bigl(p(v),\tilde{q}(v)\bigr),\\
    \mathcal{A}_{\rm RRSw} &\approx \sum_{v\in\mathrm{Top}_{n-1}}p(v)+p(x_n)+\sum_{v\notin\mathrm{Top}_n}\min\!\bigl(p(v),(1-A)\tilde{q}(v)\bigr).
\end{align*}
Hence
\begin{equation*}
    \mathcal{A}_{\text{Rheo}}-\mathcal{A}_{\rm RRSw}
    \;\approx\;
    \sum_{v\notin\mathrm{Top}_n}\Bigl[\min\!\bigl(p(v),\tilde{q}(v)\bigr)-\min\!\bigl(p(v),(1-A)\tilde{q}(v)\bigr)\Bigr]
    \;\ge\;0.
\end{equation*}
The gap is strictly positive as soon as there exists any tail token $v$ with $(1-A)\tilde{q}(v)<p(v)<\tilde{q}(v)$. In other words, \textbf{under this approximation, stochastic-first improves acceptance by preserving the full tail-draft overlap}, whereas sequential RRSw compresses the overlap by the prefix mass $A$. When the approximation is loose (i.e., $\tilde{q}$ places substantial mass outside $\mathrm{Top}_{n-1}$), the margin-token terms no longer cancel, and the dominance is not guaranteed.


\section{Example of Dynamic Tree Construction}
\label{app:tree-example}

We illustrate the EAGLE-style expand-then-rerank procedure of Section~\ref{sec:tree-sd} with expansion width $K=10$, depth $D=8$, and global budget $N=60$ (the default setting in Section~\ref{sec:exp}). The root first produces $10$ candidate tokens at the first layer. Each of these nodes is expanded into $10$ children, yielding $100$ candidates at the second layer, among which the $10$ with the highest path scores are selected for further expansion; here the path score of a node is the product of draft probabilities along its root-to-node path. This process repeats until the eighth layer, producing $10+7\times100=710$ candidate nodes in total. Finally, global reranking retains the $60$ highest-scoring nodes (including the root) across all layers; by Lemma~\ref{lem:ancestor}, the retained set is ancestor-closed and hence forms a valid tree.

Selection for expansion is distinct from final retention: a node that is not expanded may still survive global reranking, and an expanded node may still have all its children discarded. Thus, the candidates generated during local expansion do not yet constitute the final verification tree. In RheoSampling, each local candidate pool consists of the lead, sampled, and fill tokens. The fill tokens are deterministic candidates rather than additional samples from the tail distribution $\tilde{q}$, and the sampled token enters the path scores through its proxy probability $q_{\mathrm{proxy}}$ rather than its draft probability.

\section{Limitations}
\label{limitations}
While RheoSampling successfully injects stochasticity into dynamic-tree speculative decoding and provides rigorous lossless guarantees, several limitations merit discussion.

\textbf{Engineering overhead at low temperatures.}
When the sampling temperature \(T\) approaches zero, the target distribution becomes increasingly sharp, and the practical benefit of stochastic exploration diminishes. Even with \(m = 0\), RheoSampling still pays the overhead of sampling, proxy assignment, and OT-based verification that pure deterministic top-\(K\) expansion avoids. In greedy or low-temperature decoding, where the acceptance-length advantage becomes negligible, this extra cost is wasted. For such settings, standard top-\(K\) dynamic trees remain the more efficient choice, and RheoSampling is not a universal replacement.

\textbf{Sensitivity to the rheostat parameter \(m\).}
The hyperparameter \(m\) controls the number of deterministic lead tokens and critically determines whether the stochastically sampled token survives global pruning. While \(m = 0\) and \(m = 1\) were empirically identified as two attractive operating points, the optimal choice depends on the draft-target alignment and the temperature, and a universally optimal setting does not exist. This means \(m\) may require light tuning when deploying RheoSampling on a new model or task, which limits plug-and-play generality. An adaptive strategy for selecting \(m\) would help address this in future work.

\textbf{Analytical gap in dominance over RRSw.}
In Appendix~\ref{app:comparison}, we provide an approximate comparison showing that RheoVerification dominates the recursive rejection sampling without replacement (RRSw) baseline. This derivation relies on the approximation that the finite sum over the top-\( (n-1) \) tokens captures nearly all the residual mass, and the dominance is not guaranteed when this approximation is loose. Although the empirical evidence (Table~\ref{tab:acceptance_demo} and the end-to-end results) consistently favors RheoVerification in realistic settings, a precise characterization of the conditions under which RheoVerification dominates RRSw remains an open question.

\end{document}